\documentclass[ouraos,preprint]{imsart}

\usepackage{amssymb}
\usepackage[%
T1
]{fontenc}

\usepackage{mathrsfs}
\usepackage{latexsym}
\usepackage{amsmath}
\usepackage{amsthm}
\usepackage{amsfonts}
\usepackage{url}
\usepackage{enumerate}%

\usepackage{mathtools}

\usepackage{algorithm}
\usepackage{algpseudocode}
\usepackage{bm}
\usepackage{enumitem}
\RequirePackage[colorlinks,citecolor=blue,urlcolor=blue]{hyperref}

\algnewcommand\algorithmicinput{\textbf{Input:}}
\algnewcommand\Input{\item[\algorithmicinput]}
\algnewcommand\algorithmicoutput{\textbf{Output:}}
\algnewcommand\Output{\item[\algorithmicoutput]}

\makeatletter
\DeclareRobustCommand\widecheck[1]{{\mathpalette\@widecheck{#1}}}
\def\@widecheck#1#2{%
    \setbox\z@\hbox{\m@th$#1#2$}%
    \setbox\tw@\hbox{\m@th$#1%
       \widehat{%
          \vrule\@width\z@\@height\ht\z@
          \vrule\@height\z@\@width\wd\z@}$}%
    \dp\tw@-\ht\z@
    \@tempdima\ht\z@ \advance\@tempdima2\ht\tw@ \divide\@tempdima\thr@@
    \setbox\tw@\hbox{%
       \raise\@tempdima\hbox{\scalebox{1}[-1]{\lower\@tempdima\box
\tw@}}}%
    {\ooalign{\box\tw@ \cr \box\z@}}}
\makeatother

\makeatother

\startlocaldefs

\theoremstyle{plain} 
\newtheorem{theorem}{Theorem}[section]

\newtheorem{corollary}[theorem]{Corollary}

\newtheorem{lemma}[theorem]{Lemma}
\newtheorem*{lemma*}{Lemma}

\newtheorem{proposition}[theorem]{Proposition}

\theoremstyle{definition} 
\newtheorem{definition}[theorem]{Definition}
\theoremstyle{definition} 

\theoremstyle{remark} 

\theoremstyle{remark} 
\newtheorem{remark}[theorem]{Remark}

\numberwithin{equation}{section}

\newcommand{\I}{\mathcal{I}}

\newcommand{\pred}[1]{\boldsymbol{{1}}[#1]}

\renewcommand{\H}{{\cal H}}

\let\P\relax
\DeclareMathOperator{\P}{\mathbb{P}}
\DeclareMathOperator{\E}{\mathbb{E}}
\newcommand{\BP}{{\cal A}}
\newcommand{\SP}{{\cal V}}
\newcommand{\bp}{A} %
\renewcommand{\sp}{V} %
\newcommand{\bSP}{{\bar{\cal V}}}
\newcommand{\bsp}{\bar V} %
\newcommand{\teta}{\tilde\eta}
\newcommand{\tr}{\tilde r}
\newcommand{\X}{{\cal X}}
\newcommand{\Y}{{\cal Y}}

\newcommand{\N}{\mathbb{N}}
\newcommand{\R}{\mathbb{R}}

\newcommand{\nrm}[1]{\left\Vert #1 \right\Vert}
\newcommand{\beq}{\begin{eqnarray*}}
\newcommand{\eeq}{\end{eqnarray*}}
\newcommand{\beqn}{\begin{eqnarray}}
\newcommand{\eeqn}{\end{eqnarray}}

\newcommand{\abs}[1]{\left| #1 \right|}
\newcommand{\hide}[1]{}
\newcommand{\set}[1]{\left\{ #1 \right\}}
\newcommand{\eps}{\varepsilon}
\newcommand{\paren}[1]{\left( #1 \right)}
\newcommand{\sqprn}[1]{\left[ #1 \right]}
\newcommand{\err}{\mathrm{err}}
\newcommand{\serr}{\widehat{\err}}
\newcommand{\oo}[1]{\frac{1}{#1}}
\newcommand{\argmin}{\mathop{\mathrm{argmin}}}
\newcommand{\argmax}{\mathop{\mathrm{argmax}}}

\newcommand{\dist}{\rho} %
\newcommand{\g}{\gamma}
\newcommand{\ddim}{\operatorname{ddim}}

\newcommand{\diam}{\operatorname{diam}}

\newcommand{\gn}{\, | \,}
\newcommand{\bepf}{\begin{proof}}
\newcommand{\enpf}{\end{proof}}

\newcommand{\tS}{S'}

\newcommand{\tY}{Y'}
\newcommand{\tbY}{{\bm Y}'}
\newcommand{\tby}{{\bm y}'}

\newcommand{\nn}{\text{nn}}

\newcommand{\fns}{m} %
\newcommand{\rns}{M} %
\newcommand{\mss}{W} %

\renewcommand{\a}{\alpha}
\newcommand{\gnet}{{\bm X}}

\newcommand{\Ng}{t_\g}
\newcommand{\UB}{\textrm{UB}}
\newcommand{\missmass}{L}

\newcommand{\Vor}{\mathcal{V}}

\newcommand{\Borel}{\mathscr{B}}
\newcommand{\ind}{\boldsymbol{{1}}}
\newcommand{\UBC}{\text{UBC}}

\newcommand{\appref}[1]{Appendix \ref{ap:#1}}
\newcommand{\figref}[1]{Figure~\ref{FIG:#1}}
\newcommand{\reals}{\mathbb{R}}
\newcommand{\nats}{\mathbb{N}}
\newcommand{\lemref}[1]{Lemma \ref{lem:#1}}
\newcommand{\thmref}[1]{Theorem \ref{thm:#1}}
\renewcommand{\algref}[1]{Algorithm~\ref{alg:#1}}
\renewcommand{\eqref}[1]{(\ref{eq:#1})}
\newcommand{\secref}[1]{Section~\ref{SEC:#1}}

\newcommand{\remref}[1]{Remark~\ref{rem:#1}}

\newcommand{\ds}{\displaystyle}

\newcommand{\bmu}{{\bar{\mu}}}

\newcommand{\cX}{\mathcal{X}}

\newcommand*\diff{\mathop{}\!\mathrm{d}}

\newcommand{\subI}{_\mathrm{I}}
\newcommand{\subII}{_\mathrm{II}}
\newcommand{\subIII}{_\mathrm{III}}

\newcommand{\myi}{$\mathrm(\textup{i})$}
\newcommand{\myii}{$\mathrm(\textup{ii})$}

\newcommand{\bQ}[1]{\textup{{\bf Q#1}}}

\newcommand{\newname}{\textup{{\textsf{OptiNet}}}}

\newcommand{\Bcst}{Bayes-consistent}
\newcommand{\Bcstn}{Bayes consistent}
\newcommand{\Bcsy}{Bayes consistency}

\newcommand{\dismet}{\rho_{\text{dis}}}

\providecommand{\citet}[1]{\cite{#1}}
\providecommand{\citep}[1]{\cite{#1}}

\renewcommand{\gets}{:=}
\newcommand{\eor}{\scalebox{0.7}{${}^{{}_\blacktriangleleft}$}} %

\newcommand{\hst}{h^*}

\newcommand{\hh}{\hat h}

\newcommand{\bX}{{\bf X}}

\newcommand{\bY}{{\bf Y}}
\newcommand{\Uorel}{\mathscr{U}}

\newcommand{\bphi}{\bar\phi}
\newcommand{\mix}{\lambda}
\newcommand{\MS}{\mathfrak{X}}
\newcommand{\MSk}{\MS_{\kappa}}
\newcommand{\cont}{\mathfrak{c}}
\newcommand{\vectosetsym}{W}
\newcommand{\vectoset}[1]{\vectosetsym_{#1}}

\newcommand{\lrvmc}{\kappa_{\min}}

\newcommand{\muhlabel}{Z}
\newcommand{\Alg}{\textup{{\textsf{Alg}}}}
\newcommand{\discset}{D}
\newcommand{\MSD}{\MS_{|\discset|}}
\newcommand{\bS}{{\boldsymbol{S}}}
\newcommand{\essep}{ES}

\newcommand{\RVM}{\text{RVMC}}

\newcommand{\KSW}{\cite{DBLP:conf/nips/KontorovichSW17}}
\newcommand{\KSWfull}{\cite{KontorovichSW17arxiv}}
\newcommand{\KSU}{\cite{DBLP:journals/jmlr/KontorovichSU17+nips}}
\endlocaldefs

\begin{document}

\begin{frontmatter}
  
\title{Universal Bayes consistency\\ in metric spaces}
\runtitle{Universal metric Bayes consistency}

\begin{aug}
  \author{\fnms{Steve} \snm{Hanneke}\thanksref{m1}
    \ead[label=e1]{steve.hanneke@gmail.com}
  }
  \and
  \author{\fnms{Aryeh} \snm{Kontorovich}\thanksref{m2}
    \ead[label=e2]{karyeh@cs.bgu.ac.il}
  } \\
  \and
  \author{\fnms{Sivan} \snm{Sabato}\thanksref{m2}
    \ead[label=e3]{sabatos@cs.bgu.ac.il}
  }
  \and
  \author{\fnms{Roi} \snm{Weiss}\thanksref{m3}
    \ead[label=e4]{roiw@ariel.ac.il}
  }

  \runauthor{S. Hanneke, A. Kontorovich, S. Sabato, R. Weiss}
  
  \affiliation{
    Toyota Technological Institute at Chicago\thanksmark{m1}\\
    and
    Ben-Gurion University of the Negev\thanksmark{m2}\\
    and
    Ariel University\thanksmark{m3}
  }

\end{aug}

\begin{abstract}
We extend a recently proposed 1-nearest-neighbor-based multiclass learning algorithm and prove that our modification is universally strongly \Bcstn\ in all metric spaces admitting {\em any} such learner, making it an ``optimistically universal'' \Bcst\ learner. 
This is the first learning algorithm known to enjoy this property;
by comparison, the $k$-NN classifier and its variants are not generally universally \Bcstn, except under additional structural assumptions, such as an inner product, a norm, finite dimension, or a Besicovitch-type property.

The metric spaces in which universal \Bcsy\ is possible are the ``essentially separable'' ones --- a notion that we define, which is more general than standard separability.
The existence of metric spaces that are not essentially separable is widely believed to be independent of the ZFC axioms of set theory.
We prove that essential separability exactly characterizes the existence of a universal \Bcst\ learner for the given metric space.
In particular, this yields the first impossibility result for universal \Bcsy.

Taken together, our results completely characterize strong and weak universal Bayes consistency in metric spaces.
\end{abstract}

\setattribute{keyword}{AMS}{AMS 2010 subject classifications:}


\setattribute{keyword}{support}{Support:}
\begin{keyword}[class=support]
\kwd{Aryeh Kontorovich was supported in part by the Israel Science Foundation (grant No. 755/15), Paypal and IBM.
Sivan Sabato was supported in part by the Israel Science Foundation (grant No. 555/15).}
\end{keyword}

\end{frontmatter}

\section{Introduction}
\label{SEC:intro}

Since their inception in the 1950's \citep{FH1989} --- or, according to some accounts, nearly 1000 years earlier \citep{DBLP:journals/prl/Pelillo14} --- nearest-neighbor methods have provided an intuitive and reliable suite of techniques for performing classification in metric spaces.
For $k$-NN based methods, it has been generally understood that some notion of finite dimensionality is both necessary and sufficient for the methods to be \Bcstn\ under all distributions over the metric space --- a property known as {\em universal \Bcsy} (UBC). 
However, a complete characterization of the metric spaces in which {\em any} nearest neighbor method (or other learners, for that matter) is UBC has been so far unknown. 
For the problem of multiclass classification, we resolve these questions exhaustively.

To answer these questions, we study a compression-based $1$-NN algorithm for multiclass classification which was proposed in 2017 \cite{DBLP:conf/nips/KontorovichSW17}, and shown to be strongly UBC in all metric spaces of bounded diameter and doubling dimension. It was also shown that there exist infinite-dimensional spaces in which this algorithm is strongly \Bcstn, while classic $k$-NN based methods provably are not.
Left open was the full characterization of metric spaces in which this algorithm is UBC.
In this work, we provide this characterization. Moreover, we prove that this algorithm is UBC in any metric space for which a UBC algorithm exists, thus resolving the above fundamental open question about nearest-neighbor methods.

\paragraph{Main results}
We design a generalized version of the algorithm used in \cite{DBLP:conf/nips/KontorovichSW17}, which we call \newname{}
(see Algorithm~\ref{alg:simple}).
The contribution of this paper is twofold:
(i) We show that \newname{} is universally strongly \Bcstn\ in all essentially separable metric spaces. 
A formal definition of {\em essential separability} --- our broadening of the standard notion of separability ---
is given in \secref{COMPRESSION_SCHEME}.
Briefly, in an essentially separable metric space, the total mass of every probability measure is contained
in some separable subspace.
Whether {\em every} metric space is essentially separable is widely believed to hinge upon set-theoretic axioms that are independent of ZFC, having to do with the existence of certain measurable cardinals (we provide the relevant set-theoretic background in Section~\ref{SEC:RVMC_PRE}).
(ii) We show that in any set-theoretic model that allows the existence of non-essentially separable metric spaces, no (strong or weak) universally \Bcst\ learner is possible on such spaces.
To our knowledge, this is the first construction of a learning setting in which universal \Bcsy\ is impossible. 
In contrast, if one adopts a set-theoretic model in which every metric space is essentially separable, then \newname{} is always universally strongly \Bcstn.
As such, \newname{} is {\em optimistically} universally \Bcstn\ for metric spaces, in a sense analogous to \cite{DBLP:journals/corr/Hanneke17}: it succeeds whenever success is possible,
and
is the first learning algorithm known to enjoy this property.
For comparison, $k$-NN and other existing nearest-neighbor approaches are only universally \Bcstn\ under additional structural assumptions, such as an inner product, a norm, a finite dimension, or a Besicovitch-type property \cite{cerou2006nearest,MR2235289,MR2654492}, all of which are significantly stronger assumptions than essential separability.

Taken together, our results completely characterize strong and weak UBC in metric spaces.

\paragraph{Related work}
Nearest-neighbor methods were initiated by Fix and Hodges in 1951 \cite{FH1989} and, in the celebrated $k$-NN formulation, have been placed on a solid theoretical foundation \cite{CoverHart67,stone1977,devroye1996probabilistic,MR877849,DBLP:journals/corr/ChaudhuriD14}.
Following the pioneering work of \citet{CoverHart67,stone1977} on nearest-neighbor classification, it was shown by \citet{MR877849,devroye1996probabilistic,gyorfi:02} that the $k$-NN classifier is universally strongly \Bcstn\ in $(\R^d,\nrm{\cdot}_2)$.
These results made extensive use of the Euclidean structure of $\R^d$, but in \citet{shwartz2014understanding} a weak Bayes-consistency result was shown for metric spaces with a bounded diameter and finite doubling dimension, and additional distributional smoothness assumptions.

Consistency of NN-type algorithms in more general (and, in particular, infinite-dimensional) metric spaces was discussed in \cite{MR2327897,MR2235289,MR2654492,cerou2006nearest,MR1366756, forzani2012consistent}. 
Characterizations of \Bcsy\ for the standard $k$-NN \cite{cerou2006nearest,forzani2012consistent} and for a generalized ``moving window'' classification rule \cite{MR2327897} were given in terms of a Besicovitch-type condition (see Section~\ref{SEC:disc} for a more detailed discussion).
By Besicovitch's density theorem \cite{fremlin2000measure}, in $(\R^d,\nrm{\cdot}_2)$, and more generally in finite-dimensional normed spaces, the aforementioned condition holds for all distributions;
however, in infinite-dimensional spaces this condition may be violated \cite{preiss1979invalid,MR609946}.
The violation of the Besicovitch condition is not an isolated pathology --- occurring, for example, in the commonly used Gaussian Hilbert spaces \cite{MR1974687}.
Leveraging the consistency of $k$-NN in finite dimensions, the {\em filtering} technique (taking the first $d$ coordinates in some basis representation for an appropriate $d$) was shown to be universally weakly consistent in \cite{MR2235289}.
However, that technique is only applicable in separable Hilbert spaces, as opposed to more general metric spaces.
For compact metric spaces, the SVM algorithm can be made universally \Bcstn\ by using an appropriate kernel \cite{DBLP:conf/nips/ChristmannS10}.


Although the classic $1$-NN classifier is well-known to be inconsistent in general, in recent years a series of papers has presented various ways of  learning a {\em regularized} $1$-NN classifier, as an alternative to $k$-NN.
Gottlieb et al. \cite{DBLP:journals/tit/GottliebKK14+colt} showed that an approximate nearest-neighbor search can act as a regularizer, actually improving generalization performance rather than just injecting noise.
This technique was extended to multiclass classification in \cite{kontorovich2014maximum}.
In a follow-up work, \cite{kontorovich2014bayes} showed that applying Structural Risk Minimization (SRM) to a margin-regularized data-dependent bound very similar to that in \cite{DBLP:journals/tit/GottliebKK14+colt} yields a strongly \Bcst\ $1$-NN classifier in doubling spaces with a bounded diameter.

Approaching the problem through the lens of sample compression, a computationally near-optimal nearest-neighbor condensing algorithm was presented in \cite{gkn-ieee18+nips} and later extended to cover semimetric spaces \cite{gkn-jmlr17+aistats};
both were based on constructing $\gamma$-nets in spaces with a finite doubling dimension (or its semimetric analogue).
As detailed in \cite{kontorovich2014bayes}, margin-regularized $1$-NN methods enjoy a number of statistical and computational advantages over the traditional $k$-NN classifier.
Salient among these are explicit data-dependent generalization bounds, and considerable runtime and memory savings.
Sample compression affords additional advantages, in the form of tighter generalization bounds and increased efficiency in time and space.
Recently, \cite{DBLP:conf/nips/KontorovichSW17} provided evidence that this technique has wider applicability than $k$-NN methods,
by exhibiting an infinite-dimensional metric measure space where the compression-based learner is \Bcstn, while $k$-NN methods provably fail.

The work of Devroye et al.~\cite[Theorem 21.2]{devroye1996probabilistic} has implications for $1$-NN classifiers in $(\R^d,\nrm{\cdot}_2)$ that are defined based on data-dependent majority-vote partitions of the space.  
They showed that a {\em fixed} mapping from each sample size to a data-dependent partitioning rule, satisfying some regularity conditions, induces a universally strongly \Bcst\ algorithm.
This result requires the partitioning rule to have a VC dimension that grows sub-linearly in the sample size, and since this rule must be fixed in advance, the algorithm is not fully adaptive.
Theorem 19.3 ibid. proves weak consistency for an inefficient compression-based algorithm, which selects among all the possible compression sets of a certain size, and maintains a certain rate of compression relative
to the sample size.
The generalizing power of sample compression was independently discovered by \citet{warmuth86,devroye1996probabilistic}, and later elaborated upon by \citet{graepel2005pac,hk:19}.
In the context of NN classification, \citet{devroye1996probabilistic} lists various condensing heuristics (which have no known performance guarantees) and leaves open the algorithmic question of how to minimize the empirical risk over all subsets of a given size.

The margin-based technique developed in \cite{DBLP:journals/tit/GottliebKK14+colt,kontorovich2014maximum} relied on computing a minimum vertex cover.
Thus, it was not possible to make it simultaneously computationally efficient and \Bcstn\ when the number of labels exceeds two, since Vertex Cover on general graphs is an NP-hard problem.
Although one could resort to a $2$-approximation algorithm for vertex cover, 
this presents an obstruction to establishing the \Bcsy\ of the classifier.

In \cite{DBLP:journals/jmlr/KontorovichSU17+nips}, an active-learning algorithm was presented, which, across a broad spectrum of natural noise regimes, reduced the sample complexity roughly quadratically.
Along the way, this work circumvented the computational obstacle associated with computing a minimum vertex cover on a general graph: the trick was to construct a $\gamma$-net and take the majority label (more accurately, the {\em plurality} --- that is, the most frequent --- label; we shall use the more familiar terms ``majority label'' and ``majority vote'') in each Voronoi region.
The majority is determined by actively querying each region, where the number of calls depends on the density and noise level of the region.


\paragraph{Paper outline}
After setting down the definitions in \secref{def-not}, we describe in \secref{COMPRESSION_SCHEME} the compression-based $1$-NN algorithm \newname{} studied in this paper and its consistency on essentially-separable metric spaces is proved.
In \secref{nonsep} we prove that no universally \Bcst\ algorithm exists on metric spaces that are not  essentially separable. 
We conclude with a discussion in \secref{disc}.

\section{Definitions and Notation}
\label{SEC:def-not}
Our {\em instance space} is the metric probability space $(\X,\dist,\mu)$, where $\dist$ is a metric and $\mu$ is a probability measure. By definition, the Borel $\sigma$-algebra $\Borel$ supporting $\mu$ is the smallest $\sigma$-algebra containing the open sets of $\rho$.
For any $x \in \X$ and $r > 0$, denote by $B_{r}(x)$ the open ball of radius $r$ around $x$ under the metric $\dist$:
\[
B_{r}(x) = \{ x^{\prime} \in \X : \dist(x,x^{\prime}) < r \}.
\]
We consider a countable label set $\Y$.
The unknown sampling distribution is a probability measure $\bmu$ over $\X\times\Y$, with marginal $\mu$ over $\X$.
Denote by $(X,Y) \sim \bmu$ a pair drawn according to $\bmu$. 
The generalization error of a classifier $f:\X \rightarrow \Y$ is given by
\[
\err(f) := \P_\bmu[Y \neq f(X)],
\]
and its empirical error with respect to a labeled set $\tS\subseteq \X \times \Y$ is given by
\[
\serr(f, \tS) := \oo{|\tS|}\sum_{(x,y)\in \tS} \pred{y \neq f(x)}.
\]
The optimal Bayes risk of $\bmu$ is $R^*_\bmu := \inf \err(f),$ where the infimum is taken over all measurable classifiers $f:\X \rightarrow \Y$. We omit the subscript $\bmu$ when there is no ambiguity and denote the optimal Bayes risk of $\bmu$ by $R^*$.

For a labeled sequence $S=(x_i,y_i)_{i=1}^n\in(\X\times\Y)^n$ and any $x\in\X$, let $X_{\nn}(x,S)$ be the nearest neighbor of $x$ with respect to $S$ and let $Y_{\nn}(x,S)$ be the nearest neighbor label of $x$ with respect to $S$:
\beq
(X_{\nn}(x,S),Y_{\nn}(x,S)) := 
\argmin_{(x_i,y_i)\in S} 
\dist(x, x_i),
\eeq
where ties are broken lexicographically --- i.e., the smallest $x_i$ is chosen, with respect to a fixed total ordering of the space $\X$ (such an ordering can always be chosen to be measurable, see \appref{total_order}).
The $1$-NN classifier induced by $S$ is defined as $h_{S}( x ) := Y_{\nn}(x,S)$.
For any $m \in \nats$, any sequence $\bm X = \{x_1,\ldots,x_{m}\} \in \X^{m}$ induces a {\em Voronoi partition} of $\X$, $\Vor(\bm X) := \{V_1(\bm X),\dots, V_{m}(\bm X)\}$, where each Voronoi cell is
\beq
V_i(\bm X) := \set{x \in \X : i = \argmin_{1\le j\le m} \rho(x,x_j) },
\eeq
again breaking ties lexicographically.
In particular, for $\bm X = \{ X_i : (X_i,Y_i) \in S \}$, we have $h_S(x) = Y_i$ for all $x\in V_i(\bm X)$.

A $1$-NN algorithm is a mapping from an i.i.d.~labeled sample $S_n \sim \bmu^n$ to a labeled set $\tS_n \subseteq \X \times \Y$, yielding the $1$-NN classifier $h_{\tS_n}$. 
While the classic $1$-NN algorithm sets $\tS_n \gets S_n$, the algorithm which we analyze chooses $\tS_n$ adaptively.
More generally, a learning algorithm $\Alg$ is a mapping (possibly randomized) from a labeled sequence $S_n=(x_i,y_i)_{i=1}^n\in(\X\times\Y)^n$ to $\Alg(S_n)\in\Y^\X$, satisfying some natural measurability requirements spelled out in Remark~\ref{rem:meas-alg} below.
We say that $\Alg$ is {\em strongly \Bcstn} under $\bmu$ if $\err(\Alg(S_n))$ converges to $R^*$ almost surely,
\[
\P\!\left[\lim_{n \rightarrow \infty} \err(\Alg(S_n)) = R^*\right] = 1.
\]
Similarly, $\Alg$ is {\em weakly \Bcstn} under $\bmu$ if $\err(\Alg(S_n))$ converges to $R^*$ in expectation,
\[
\lim_{n\to\infty}\E[\err(\Alg(S_n))]=R^*.
\]
Obviously, the former implies the latter. We say that $\Alg$ is \emph{universally \Bcstn} on a metric space if $\Alg$ is \Bcstn\ for every distribution supported on its Borel $\sigma$-algebra $\Borel$.
Specializing to \newname{}, we have $\Alg(S_n)=h_{\tS_n}$.

For $A\subseteq\X$ and $\g>0$, a $\g$-{\em net} of $A$ is any {\em maximal set} $B\subseteq A$ in which all interpoint distances are at least $\g$.
In separable metric spaces, all $\g$-nets are at most countable.
Denote the diameter of a set $A \subseteq \X$ by $\diam(A)\in[0,\infty]$. 
For a partition $\cal{A}$, $\diam(\cal{A})$ denotes the maximum diameter $\diam(A)$ among all cells $A \in \cal{A}$.

For $n\in\N$, define $[n] := \{1,\ldots,n\}$.
Given a labeled set $S_n = (x_i,y_i)_{i \in [n]}$, $d \in [n]$, and any $\bm i = \{i_1,\ldots,i_d\} \in [n]^d$, denote the sub-sample of $S_n$ indexed by $\bm i$ by $S_n(\bm i) := \{(x_{i_1}, y_{i_1}), \dots,(x_{i_d}, y_{i_d})\}$.
Similarly, for a vector $\tby = \{y'_1,\ldots,y'_d\} \in \Y^d$, define $S_n{(\bm i, \tby)} := \{(x_{i_1}, y'_{1}), \dots,(x_{i_d}, y'_{d})\}$, namely the sub-sample of $S_n$ as determined by $\bm i$ where the labels are replaced with $\tby$.
Lastly, for $\bm i,\bm j \in [n]^d$, we denote $S_n(\bm i; \bm j) := \{(x_{i_1}, y_{j_1}), \dots,(x_{i_d}, y_{j_d})\}.$

We use standard order-of-magnitude notation throughout the paper;
thus, for $f,g:\N\to[0,\infty)$ we write $f(n)\in O(g(n))$ to mean $\limsup_{n\to\infty} f(n)/g(n)$ $<\infty$ and $f(n)\in o(g(n))$ to mean $\limsup_{n\to\infty} f(n)/g(n)=0$.
Likewise, $f(n)\in\Omega(g(n))$ means that $g(n)\in O(f(n))$.
In accordance with common convention, we often use the less precise notation $f(n)=O(g(n))$, etc.

The main notations are summarized in Table.~\ref{tab:notation}; some are introduced in later sections.

\begin{table}[t]%
\begin{tabular}{l|l}
Symbol & Brief description 
\\
  \hline
  $(\X, \rho, \mu)$ & metric probability space \\
  $\Borel$ & Borel $\sigma$-algebra induced by $\rho$  \\
  $B_r(x)$ & open ball of radius $r$ around $x$  \\
  $\err(f)$ & generalization error of $f:\X \rightarrow \Y$  \\
  $\serr(f, \tS)$ & empirical error of $f:\X \rightarrow \Y$ on $\tS$  \\
  $R^*$ & Bayes risk \\
$S_n = (\bm X_n,\bm Y_n)$ & random sample of size $n$ 
\\
$S_n(\bm i, \bm j)$ & subsample $(\bm X_{\bm i}, \bm Y_{\bm j})$ of $S_n$ indexed by $\bm i$ and $\bm j$ 
\\
$S_n(\bm i)$ & subsample $S_n(\bm i,\bm i)$ 
\\
$S_n(\bm i, *)$ & subsample $(\bm X_{\bm i},\bm Y^*)$ with true majority vote labels 
  \\
$\bm X(\g)$ & $\g$-net of $\bm X_n$ 
\\
$S_n(\g)$ & subsample $(\bm X(\g),\bm Y(\g))$ with empirical majority votes
\\
$\rns_n(\g)=2|\bm X(\g)|$ & size of the compression 
\\
$h_S$ & $1$-NN classifier induced by the labeled set $S$
\\
$\a_n(\g)$ & empirical error of $h_{S_n(\g)}$ on $S_n$ 
\\
$\UB_{\g}(A)$ & $\g$-envelope of $A\subseteq \X$ 
\\
$\missmass_\g(A)$ & $\g$-missing mass of $A\subseteq \X$ 
\\
$\SP(\bm X)$ & Voronoi partition of $\X$ induced by $\bm X$
\\
$\Alg(S) = \hh_S$ & Classifier obtained by learning algorithm $\Alg$ when given sample $S$
\end{tabular}
\\[4pt]
\caption{Symbols guide}
\label{tab:notation}


\end{table}

\section{Universal \Bcsy\ in separable metric spaces}
\label{SEC:COMPRESSION_SCHEME}
\begin{algorithm}[t]
\caption{(\newname) The $1$-NN compression-based algorithm}
\label{alg:simple} 
\begin{algorithmic}[1]  
\Input sample $S_n = (X_i, Y_i)_{i\in[n]}$, confidence $\delta \in (0,1)$
\Output A $1$-NN classifier
\State let $\Gamma\gets(\set{\rho(X_i,X_j) : i,j \in [n]} \cup \{\infty\}) \setminus \{0\}$
\For{$\g \in \Gamma$}
\State let $\bm X(\g)$ be a $\g$-net of $\{X_1,\ldots,X_n\}$
\State let $\rns_n(\g) \gets 2|\bm X(\g)|$
\State for each $i \in [\rns_n(\g)/2]$, let $\tY_i(\g)$ be the most frequent label in $V_i(\bm X(\g))$ as in \eqref{maj} 
\State set $\tS_n(\g) \gets (\bm X(\g), \tbY(\g))$
\EndFor
\State Set $\a_n(\g) \gets \serr(h_{\tS_n(\g)}, S_n)$
\State find $\g^*_n \in  \argmin_{\g\in\Gamma} Q(n,\a_n(\g), \rns_n(\g), \delta)$, where $Q$ is defined in \eqref{KSUbound}
\State set $\tS_n \gets \tS_n({\g^*_n})$
\State \Return $h_{\tS_n}$   
\end{algorithmic}%
\end{algorithm}
In this section we describe a variant of the $1$-NN majority-based compression algorithm developed in the series of papers \cite{DBLP:journals/jmlr/KontorovichSU17+nips,DBLP:conf/nips/KontorovichSW17,KontorovichSW17arxiv}, adapted to maintain measurability in potentially non-separable metric spaces.
We show that this variant is universally \Bcstn\ in all separable metric spaces, and the extension
to essential separability is immediate, as will become clear below.

Our variant, \newname{}, is formally presented in \algref{simple}.
It operates as follows. 
The input is the sample $S_n$; the set of points in the sample is denoted by $\bm X_n = \{X_1,\ldots,X_n\}$. The algorithm defines a set $\Gamma$ of all scales $\g > 0$ which are interpoint distances in $\bm X_n$, and the additional scale $\g = \infty$. 
For each scale in $\Gamma$, the algorithm constructs a $\g$-net of $\bm X_n$;
note that any singleton in $\bm X_n$ is an $\infty$-net.
Denote the constructed $\g$-net by $\bm X(\g) := \{X_{i_1},\ldots,X_{i_{M/2}}\}$, where $\rns/2\equiv\rns_n(\g)/2 :=|\gnet(\g)|$ denotes its size and $\bm i \equiv \bm i(\g) := \{i_1,\ldots,i_{\rns/2}\} \in [n]^{\rns/2}$ denotes the indices selected from $S_n$ for this $\g$-net.  For each $\g$-net, \newname{} finds the empirical majority vote labels in the Voronoi cells defined by the partition $\Vor(\bm X(\g)) = \{V_1(\bm X(\g)),\ldots,V_{\rns/2}(\bm X(\g))\}$; 
these labels are denoted by $\tbY(\g)\in \Y^{\rns/2}$.  
Formally, for $i \in [\rns/2]$,
\beqn\label{eq:maj} 
\tY_i(\g) := \argmax_{y \in \Y} |\{ j \in [n] : X_j \in V_{i}(\bm X(\g)), Y_j = y\}|,
\eeqn 
where ties are broken based on a fixed preference order on the countable set $\Y$.
The result of the procedure is a labeled set $\tS_n(\g) := S_n(\bm i (\g), \tbY(\g))$ for every possible scale $\g \in \Gamma$.
The algorithm then selects one scale $\g^*\equiv\g_n^*$ from $\Gamma$, and outputs the hypothesis that it induces, $h_{\tS_n(\g^*)}$. The choice of $\g^*$ is based on minimizing a generalization error bound, denoted $Q$, which upper bounds $\err(h_{\tS_n(\g)})$ with high probability.
The error bound is derived based on a compression-based analysis, as follows.

For an even integer $m \leq 2n$, we say that a specific $\tS_n$ is an \emph{$(\alpha,m)$-compression} of $S_n$ if there exist $\bm i, \bm j \in [n]^{m/2}$ such that $\tS_n = S_n(\bm i, \bm j)$ and $\serr(h_{\tS_n},S_n) \leq \alpha$.
Note that at most $m$ examples from $S_n$ determine $h_{\tS_n}$, hence this is a compression scheme of size at most $m$. 


The papers \KSW\ and \KSWfull\ give a consistency result for the original algorithm of \KSU, on metric spaces with a finite doubling dimension and a finite diameter, under the following assumptions on the generalization error bound $Q(n,\alpha,m,\delta)$:
\begin{enumerate}[label=\textbf{\bQ{}\arabic*.},leftmargin=*]
\item  For any $n\in\N$ and $\delta \in (0,1)$, with probability at least $1 - \delta$ over $S_n \sim \bmu^n$, for all $\a \in [0,1]$ and even $\fns  \in [2n]$:
If $\tS_n$ is an  $(\a, \fns)$-compression of $S_n$, then 
$$
\err(h_{\tS_n}) \leq Q(n,\alpha, m,\delta).
$$
\item For any fixed $n \in \N$ and $\delta \in (0,1)$, $Q$ is monotonically increasing in $\a$ and in $\fns$.
\item There is a sequence $\{\delta_n\}_{n = 1}^\infty$,   $\delta_n \in (0,1)$ such that $\sum_{n=1}^\infty \delta_n < \infty$, and for all $m$,
\[
\lim_{n\rightarrow \infty} \sup_{\a\in [0,1]} (Q(n,\a,m,\delta_n) - \a) = 0.
\]
\end{enumerate}
\newcommand{\Qthreeb}{\textbf{\bQ{3$^\prime$}}}

Here, we provide a consistency result that holds for more general metric spaces.
We prove that \newname{} is universally strongly \Bcstn\ in all \emph{essentially separable} metric spaces. 
Recall that $(\X,\rho)$ is {\em separable} if it contains a dense countable set.
A metric probability space $(\X,\rho,\mu)$ is separable if there is a measurable $\X'\subseteq\X$ with $\mu(\X')=1$ such that $(\X',\rho)$ is separable.
We will call a metric space $(\X,\rho)$ {essentially separable} (\essep) if, for \emph{every} probability measure $\mu$ on $\Borel$, the metric probability space $(\X,\rho,\mu)$ is separable.

To prove this stronger result, we require a slightly stronger version of property \bQ3.
\begin{enumerate}[label=\textbf{\bQ{}\arabic*.},leftmargin=*]
\item[\Qthreeb\bf.] There is a sequence $\{\delta_n\}_{n = 1}^\infty$, $\delta_n \in (0,1)$ such that $\sum_{n=1}^\infty \delta_n < \infty$, and for any sequence $\fns_n \in o(n)$,
\[
\lim_{n\rightarrow \infty} \sup_{\a\in [0,1]} (Q(n,\a,\fns_n,\delta_n) - \a) = 0.
\]
\end{enumerate}
Property \Qthreeb\ is slightly stronger than \bQ3, since it allows $m$ to grow as $o(n)$ instead of keeping it as a constant.
The compression bound used in \cite{DBLP:conf/nips/KontorovichSW17} does not satisfy this property, since it includes a term of the order $m\log(n)/(n-m)$. Therefore, if $m_n = \Omega(n/\log(n))$, then $m_n = o(n)$, yet this term does not converge to zero for $n \rightarrow \infty$, thus precluding consistency of the algorithm in \cite{DBLP:conf/nips/KontorovichSW17} for such cases.
We provide here a tighter compression bound, which does satisfy \Qthreeb.
\begin{lemma}\label{lem:KSUbound}
For $m \leq n-2$, define
\begin{align}
\label{eq:KSUbound}
Q(n,\alpha,m,\delta) := 
& \frac{n}{n-m} \alpha + 
\sqrt{\frac{8 (\frac{n}{n-m})\alpha\big(m \ln( 2 e n / m ) + \ln(2n/\delta)\big) }{n-m}} 
\\ & + \frac{9\big( m \ln( 2 e n / m ) + \ln(2n/\delta)\big)}{n-m}.
\notag
\end{align}
For $m > n-2$, define $Q(n,\alpha,m,\delta) := \max(1,Q(n,\alpha,n-2,\delta))$.
Then the function $Q$ satisfies the properties \bQ1, \bQ2, \Qthreeb.
\end{lemma}

The approach to obtaining property \Qthreeb\ is inspired by refinements of compression-based generalization bounds holding for the special case of compression schemes which have a \emph{permutation-invariant} reconstruction function \cite{graepel2005pac}.
While $h_{\tS_n}$ cannot quite be expressed as a permutation-invariant function of a subset of the $(X_i,Y_i)$ data points, it \emph{can} be expressed as a function that is invariant to permutations of two subsets of $(X_i,Y_i)$ points. This is used in the proof of \lemref{KSUbound}, which is provided in \appref{KSUbound}, to derive the tighter compression bound in \eqref{KSUbound}. 
This bound is derived using Bernstein's inequality over $n-m$ random variables and applying a union bound over all ${n\choose m/2}^2$, $1\leq m/2\leq n$, possible compressions.

Our main technical innovation, which allows us to dispose of the finiteness requirements on the dimension and the diameter of the metric space that were assumed in \cite{DBLP:conf/nips/KontorovichSW17}, is the sublinear growth of $\g$-nets.
Another straightforward but crucial insight is to approximate functions in $L^1(\mu)\gets \{f: \int {\abs{f}} \diff\mu < \infty\}$ by {\em Lipschitz} ones, rather than by continuous functions with compact support as in \cite{DBLP:conf/nips/KontorovichSW17}.
The latter approximation requires local compactness, which essentially amounts to a finite dimensionality condition. 
Our new approach does not require local compactness or finite dimensionality.

\begin{theorem}
\label{thm:comp-consist}
Let $(\X,\rho,\mu)$  be a separable metric probability space.
Let $Q$ be a generalization bound that satisfies Properties \bQ1, \bQ2, \Qthreeb, and let $\delta_n$ be as stipulated by \Qthreeb.
If the input confidence $\delta$ for input size $n$ is set to $\delta_n$, then the $1$-NN classifier $h_{\tS_n({\g^*_n})}$ calculated by \newname{} is strongly \Bcstn\ on $(\X,\rho,\mu)$:
\[
\P[\lim_{n \rightarrow \infty} \err(h_{\tS_n(\g_n^*)}) = R^*] = 1. 
\]
\end{theorem}

\begin{remark}
\label{rem:cross_validation}
\newname{} selects the scale $\g$ based on a compression bound.
This creates a close connection between the algorithm and the proof of consistency below.
However, it is worth noting that it is possible instead to choose $\g$ based on a hold-out validation set: for instance, using $n/2$ of the $n$ samples to construct the predictor for each possible $\g$ value, and then from among these values $\g$, one can select the $\g$ whose predictor makes the smallest number of mistakes on the remaining $n/2$ samples.
Since the analysis of \KSW{} (see \KSWfull), and its generalization below, show that there exists a choice of $\g^*$ for each $n$ such that \newname{} is \Bcstn, this alternative technique of selecting $\g$ based on a hold-out sample would only lose an additive $O\!\left(\sqrt{{\log(n)}/{n}}\right)$ compared to using that $\g^*$, and hence would also be \Bcstn.
\eor\end{remark}

\begin{remark}
\label{rem:efficiency}
\newname{} is computationally efficient. Using a farthest-first-traversal procedure such as Algorithm 1 in \cite{kpotufe2017time}, one can 
construct the $\gamma$-nets simultaneously for all $\gamma$ values, including their corresponding empirical errors, 
in $O(n^2)$ time, leading to a total runtime of $O(n^2)$.
\eor\end{remark}

Given a sample $S_n\sim \bmu^n$, we abbreviate the optimal empirical error $\a_n^*=\a(\g^*_n)$ and the optimal compression size $\rns_n^*=\rns(\g^*_n)$ as computed by \newname{}.
As discussed above, the labeled set $\tS_n(\g_n^*)$ computed by \newname{} is a $(\a_n^*, \rns_n^*)$-compression of the sample $S_n$.
For brevity we denote
\[
Q_n(\alpha,\fns) := Q(n,\alpha,\fns,\delta_n).
\]

To prove Theorem \ref{thm:comp-consist}, we first follow the standard technique, used also in \KSWfull, of decomposing the excess error over the Bayes error into two terms:
\beq
\err(h_{\tS_n(\g^*_n)}) - R^*  
&= &
\big(\err(h_{\tS_n(\g^*_n)}) - Q_n(\a_n^*,\rns_n^*) \big)
+
\big(Q_n(\a_n^*,\rns_n^*) - R^*\big)
\\
&=:&
T\subI(n) + T\subII(n).
\eeq
We now show that each term decays to zero almost surely.
For the first term, $T\subI(n)$, we have, similarly to \KSWfull, that Property \bQ1\ implies that for any $n> 0$,
\beqn
\label{eq:termI_bound}
\P\!
\left[
\err(h_{\tS_n(\g^*_n)}) - Q_n(\a_n^*,\rns_n^*) > 0
\right] 
\leq \delta_n.
\eeqn
Based on the Borel-Cantelli lemma and the fact that $\sum \delta_n <\infty$, we have that $\limsup_{n\to\infty} T\subI(n) \leq 0$ with probability $1$. 

The main difference from the proof in \KSWfull\ is in the argument for establishing  $\limsup_{n\to\infty} T\subII(n) \leq 0$ almost surely. 
We now show that the generalization bound $Q_n(\a_n^*, \rns_n^*)$ also approaches the Bayes error $R^*$, thus proving $\limsup_{n\to\infty} T\subII(n) \leq 0$
almost surely.


We will show below that there exist $N = N(\eps) > 0$, $\g = \g(\eps) > 0$, and universal constants $c,C>0$ such that $\forall n \geq N$, 
\beqn
\label{eq:termII_bound_fixed_g}
\P[Q_n(\a_n(\g), \rns_n(\g))>R^* + \eps]
\leq
C n e^{-c n\eps^2} + 1/n^2.
\eeqn
For any $\g > 0$ (even if $\g \notin \Gamma$), \newname{} finds $\g_n^*$ such that
\beq
Q_n(\a_n^*, \rns_n^*) &=& \min_{\g' \in \Gamma} Q_n(\a_n(\g'),\rns_n(\g'))
\,\leq\, Q_n(\a_n(\g), \rns_n(\g)).
\eeq
The bound in (\ref{eq:termII_bound_fixed_g}) thus implies that $\forall n \geq N$,
\beqn
\label{eq:termII_bound}
\P[Q_n(\a_n^*, \rns_n^*) > R^* + \eps] 
\leq  C n e^{-cn\eps^2} + 1/n^2.
\eeqn
By the Borel-Cantelli lemma, this implies that almost surely,
\beq
\limsup_{n\rightarrow \infty} T\subII(n) = \limsup_{n\rightarrow \infty} (Q_n(\a_n^*, \rns_n^*) - R^*) \leq 0.
\eeq 
Since $\forall n, T\subI(n) + T\subII(n) \geq 0$, this implies $\lim_{n\to\infty} T\subII(n) = 0$ almost surely, thus completing the proof of \thmref{comp-consist}. 

It remains to prove (\ref{eq:termII_bound_fixed_g}). We note that a simpler form of \eqref{termII_bound_fixed_g} is proved in \KSWfull, where they relied on the finiteness of the dimension and the diameter of the space to upper bound the compression size $M_n(\gamma)$ with probability $1$. 
For $A \subseteq \X$, denote its $\g$-envelope by $\UB_\g(A)$ $:= \cup_{x\in A} B_\g(x)$ and consider the $\g$-{\em missing mass} of $S_n$, defined as the following random variable:
\beqn
\label{eq:Lg}
\missmass_\g(S_n) := \mu( \X \setminus \textup{\UB}_\g(S_n)).
\eeqn
We bound the left-hand side of (\ref{eq:termII_bound_fixed_g}) using a function $n\mapsto \Ng(n)$ of order $o(n)$, used to upper bound the compression size; 
$\Ng$ 
will be specified below.
\beqn
\label{eq:split_miss}
&& \P[ Q_n(\a_n(\g),\rns_n(\g)) > R^* + \eps ]
\\\nonumber
& \leq &
\P
\left[ Q_n(\a_n(\g),\rns_n(\g)) > R^* + \eps 
\;\wedge\; \missmass_\g(S_n) \leq \frac{\eps}{10}
\;\wedge\; \rns_n(\g) \leq \Ng(n)
\right]
\\
\nonumber
&& \,
+\, \P[ \missmass_\g(S_n) > {\eps}/{10} ]
+ \P[ \rns_n(\g) > \Ng(n)] \\
\nonumber
&=:& P\subI+P\subII+P\subIII.
\eeqn
First, we bound $P\subI$. By a union bound,
\begin{align*}
\label{eq:P_Q_sum}
&\P
\left[ Q_n(\a_n(\g),\rns_n(\g)) > R^* + \eps 
\;\wedge\; \missmass_\g(S_n) \leq \frac{\eps}{10}
\;\wedge\; \rns_n(\g) \leq \Ng(n)
\right]
\\
&\leq 
\sum_{d=1}^{\Ng(n)} \P
\Big[
Q_n(\a_n(\g),\rns_n(\g)) > R^* + \eps
\;\wedge\; \missmass_\g(S_n) \leq \frac{\eps}{10}
\;\wedge\; \rns_n(\g)=d
\Big].
\end{align*}
Thus, it suffices to bound each term in the right-hand sum separately. We do so in the following lemma.

\begin{lemma} \label{lem:boundqd} There exists a function $\eps \mapsto \g(\eps)$ for $\eps > 0$, such that under the conditions of \thmref{comp-consist}, there exists an $n_0$ such that for all $n \geq n_0$, and for all $d \in [\Ng(n)]$, letting $\g := \g(\eps)$, 
\[
p_d := \P
\Big[
Q_n(\a_n(\g),\!\rns_n(\g)) > R^* \!+ \eps
\;\wedge\; \missmass_\g(S_n) \!\leq\! \frac{\eps}{10}
\;\wedge\; \rns_n(\g)\!=\!d\Big] \leq e^{-\frac{n\eps^2}{32}}\!.
\]
\end{lemma}
Applying \lemref{boundqd} and summing over all $1\leq d\leq \Ng(n)$, we have that, for $n$ large enough so that $\Ng(n) \leq n$,
\beqn
\label{eq:first_term}
P\subI
\le
\sum_{d=1}^{\Ng(n)} p_d \;\leq\; \Ng(n)  e^{ -\frac{n \eps^2}{32} } \leq n e^{ -\frac{n \eps^2}{32} }.
\eeqn

\lemref{boundqd} is a generalization of Lemma 10 in \KSWfull. The main difference is that Lemma 10 holds in doubling spaces and uses the fixed map $\Ng(n)= 2\left\lceil{\diam(\X)}/{\g}\right\rceil^{\ddim}$ for all $n\in\N$. The proof of \lemref{boundqd} is the same as that of Lemma 10 in \KSWfull, except for two changes that adapt it for a general metric space.
First, where Lemma 10 uses the fact that $\Ng$ is set to a constant function and thus $\lim_{n \rightarrow \infty}\Ng(n)/n = 0$, the proof of \lemref{boundqd} uses instead the property that $\Ng(n) = o(n)$, which again leads to the same limit.

In addition, the proof of \lemref{boundqd} employs a new result, \lemref{richness} given below, instead of Lemma 8 from \KSWfull.
Lemma 8 from \KSWfull\ states that for metric spaces with a finite doubling dimension and diameter, Bayes error $R^*$ can be approached using classifiers defined by the true majority-vote labeling over fine partitions of $\X$.
Here, we prove that this holds for general metric spaces. Let $\SP = \{\sp_1,\dots\}$ be a countable partition of $\X$, and define the function $I_\SP: \X \to \SP$ such that $I_\SP(x)$ is the unique $\sp\in\SP$ for which $x\in\sp$.
For any measurable set $\emptyset\neq E\subseteq\X$ define the true majority-vote label $y^*(E)$ by
\beqn
\label{eq:bar y(E)}
y^*(E)
= \argmax_{y\in\Y} \P(Y=y \gn X\in E),
\eeqn
where ties are broken lexicographically.
Given $\SP$ and a measurable set $\mss\subseteq \X$, define the true majority-vote classifier $h_{\SP,\mss}^*:\X\to\Y$ given by
\beqn
\label{eq:Strue2}
h_{\SP,\mss}^*(x) = y^*(I_\SP(x) \cap\mss).
\eeqn
The new lemma can now be stated as follows. 
\begin{lemma}
\label{lem:richness}
Let $\bmu$ be a probability measure on $\X\times\Y$, where $\X$ is a metric probability space.
For any $\nu>0$, there exists a diameter $\beta=\beta(\nu)>0$ such that for any countable measurable partition $\SP = \{\sp_1,\dots\}$ of $\X$ and any measurable set $\mss\subseteq \X$ satisfying 
\begin{itemize}
\item[\myi] $\mu(\X\setminus\mss) \leq \nu$
\item[\myii] $\diam(\SP \cap\mss)\leq \beta$,
\end{itemize}
the true majority-vote classifier $h_{\SP,\mss}^*$ defined in (\ref{eq:Strue2}) satisfies
\beq
\err(h_{\SP,\mss}^*) \leq R^* + 5\nu.
\eeq
\end{lemma}
The proof of \lemref{richness} is identical to that of Lemma 8 from \KSWfull, except for the following change:
in the proof of Lemma 8 from \KSWfull, they use their Lemma 7, which states that on doubling spaces, the set of continuous functions with compact support is dense in $L^1(\mu)$.
To remove the requirement of compact support, which restricts the type of spaces for which this lemma holds, we use instead a stronger approximation result, which states that Lipschitz functions are dense in $L^1(\mu)$ for any metric probability space.
For completeness, we include a proof of this fact in Lemma~\ref{lem:dense_cont} in the supplementary material \cite{HKSW_sup}, where a complete proof of \lemref{richness} is also given.

Having established \lemref{richness}, this completes the necessary generalizations to obtain \lemref{boundqd}, whose proof is given in \appref{boundqd} for completeness.
This proves the bound on $P\subI$ claimed in \eqref{first_term}.

We now turn to constructing the function $\Ng$, which bounds the compression size (i.e., twice the $\g$-net size) with high probability, and bounding $P\subII$ and $P\subIII$. 
\begin{lemma}
\label{lem:sublinear_comp}
Let $(\X,\rho,\mu)$ be a separable metric probability space.
For $S_n\sim\mu^n$, let $\gnet(\g)$ be any $\g$-net of $S_n$.
Then, for any $\g>0$, there exists a function $\Ng:\N\to\R_+$ in $o(n)$ such that
\beqn
\label{eq:sublinear_comp}
\P\left[
\sup_{\g\text{-}\mathrm{nets}\,\,\gnet(\g)}
2|\gnet(\g)| \geq \Ng(n)
\right]
\leq 
1/n^2.
\eeqn
\end{lemma}
This result can be compared to the case of finite-dimensional and finite-diameter metric spaces, in which one can set $\Ng(n) := 2\left\lceil \frac{\diam(\X)}{\g}\right\rceil^{\ddim}$ for all $n\in\N$, where $\ddim$ is the (finite) doubling dimension and $\diam(\cX)$ is the diameter of the space, and get that $\P[\rns_n(\g) \geq \Ng(n)] = 0$. 
The proof of \lemref{sublinear_comp} is provided in \appref{proofs} of the supplementary material \cite{HKSW_sup}.

This lemma implies that $P\subIII\le1/n^2$, while a bound on $P\subII$, which bounds the $\gamma$-missing-mass $\missmass_\g(S_n)$, is furnished by the following lemma, whose proof is given in \appref{proofs} of the supplementary material \cite{HKSW_sup}:
\begin{lemma}
\label{lem:missing_mass}
Let $(\X,\rho,\mu)$ be a separable metric probability space, $\g>0$ be fixed, and the $\g$-missing mass $\missmass_\g$ defined as in (\ref{eq:Lg}).
Then there exists a function $u_\g:\N\to\R_+$ in $o(1)$, such that for $S_n \sim\mu^n$ and all $t>0$,
\beqn
\label{eq:miss_mass_conc}
\P\left[
\missmass_\g(S_n) \geq u_\g(n) + t
\right]
\leq 
\exp\left(-  n t^2\right).
\eeqn
\end{lemma}


Taking $n$ sufficiently large so that $u_\g(n)$, as furnished by Lemma~\ref{lem:missing_mass}, satisfies $u_\g(n) \leq \eps/20$, and invoking Lemma~\ref{lem:missing_mass} with $t=\eps/20$, we have
\beqn
\label{eq:Psi_n}
P\subII = \P[\missmass_\g(S_n) > \eps/10] \leq e^{-\frac{n \eps^2}{400}}.
\eeqn
Plugging \eqref{first_term}, \eqref{Psi_n}, and $P\subIII \leq 1/n^2$ into \eqref{split_miss}, we get that (\ref{eq:termII_bound_fixed_g}) holds, which completes the proof of \thmref{comp-consist}. 

\section{Essential separability is necessary for universal Bayes consistency}
\label{SEC:nonsep}
Recall that a metric space $(\X,\rho)$ is {\em essentially separable} (\essep) if for every probability measure $\mu$ on the Borel $\sigma$-algebra $\Borel$, the metric probability space $(\X,\rho,\mu)$ is separable;
namely, there is an $\X'\subseteq\X$ with $\mu(\X')=1$ such that $(\X',\rho)$ is separable.
In Theorem~\ref{thm:comp-consist}, we established that \newname{} is indeed universally \Bcstn\ (\UBC) for all such spaces.
As such, essential separability of a metric space is sufficient for the existence of a \UBC\ learning rule in that space.
In this section, we show that essential separability is also necessary for such a rule to exist.

The metric spaces one typically encounters in statistics and machine learning are all \essep, as reflected by Dudley's remark  that ``for practical purposes, a probability measure defined on the Borel sets of a metric space is always concentrated in some separable subspace'' \cite{dudley1999uniform}.
The question of whether non-\essep\ metric spaces exist at all turns out to be rather subtle.
It is widely believed that the existence of non-\essep\, spaces is \emph{independent of the ZFC axioms of set theory} (see \secref{RVMC_PRE} for further details).
In other words, it is believed that, assuming that ZFC is consistent, its axioms neither necessitate nor preclude the existence of non-\essep\ metric spaces.

The main contribution of this section is to show that in any non-\essep\ metric space (if one exists), no learning rule is UBC.

\begin{theorem}
\label{thm:equivalence}
Let $(\X,\rho)$ be a non-\essep\ metric space equipped with the Borel $\sigma$-algebra $\Borel$.
Then no (weak or strong) \UBC\ algorithm exists on $(\X,\rho)$.
\end{theorem}

Combining this result with \thmref{comp-consist}, the following result is immediate, revealing that \newname{} is \emph{optimistically} \UBC\ (adopting the terminology of \cite{DBLP:journals/corr/Hanneke17}), in the sense that the only required assumption on $(\X,\rho)$ is that UBC learning is \emph{possible}.

\begin{corollary}
\label{cor:optimistic}
\newname{} is \UBC\ in every metric space for which there exists a \UBC\ learning rule.
\end{corollary}

\begin{remark}
Theorem~\ref{thm:equivalence} is somewhat unusual, in that it identifies a setting in which no universal \Bcst\ procedure exists.
To our knowledge, this is the first such impossibility result.
Also unusual, for a statistics paper, is the appearance of esoteric set theory.
See \cite{Ben-David2019} for another recent result discussing a setting in which learnability is independent of ZFC.\eor
\end{remark}  

In the next section we provide necessary preliminaries. \thmref{equivalence} is proved in \secref{UBC_IMPOSSIBLE}.

\subsection{Preliminaries}
\label{SEC:RVMC_PRE}
We collect necessary definitions and known results about non-\essep\ metric spaces.
In particular, we connect the existence of non-\essep\ metric spaces with the existence of \emph{real-valued measurable cardinals} (Definition \ref{def:rvmc} below).
A thorough treatment of the latter, including most of the material in this subsection, can be found in \cite{jech};
a more gentle introduction to the subject can be found in \cite{hrbacek1999introduction}.
Throughout the following presentation, we work under the standard Zermelo-Fraenkel set theory together with the Axiom of Choice, commonly abbreviated as ZFC.

\paragraph{Cardinals}
We denote the cardinality of a set $A$ by $|A|$. The first infinite (countable) cardinal is denoted by $\aleph_0=|\omega|$, where $\omega$ is the set of all finite cardinals. In particular, $\aleph_0$ is the cardinality of the set of natural numbers, $\aleph_0=|\N|$.
We write $[A]^{n}$ to denote the family of all subsets of $A$ of size $n\in\N$, and $[A]^{<\omega} \gets \bigcup_{n\in\N} [A]^{n}$ is the family of all finite subsets of $A$. 
The smallest uncountable cardinal is denoted by $\aleph_1$. The cardinality of the real numbers, also known as the \emph{continuum}, is $\cont=|\R|$. 
It is well known that $\cont=2^{\aleph_0} \geq \aleph_1 >\aleph_0$.
The Continuum Hypothesis states that  $\cont = \aleph_1$. It is known  that its truth value is independent of ZFC, so that either the Continuum Hypothesis or its negation can be added as an axiom to ZFC set theory while maintaining its consistency status.
In the following we do not include the Continuum Hypothesis (or its negation) in our set theory; thus, our discussion includes models of ZFC in which $\cont > \aleph_1$.

\paragraph{Non-trivial probability measures} 
Let $(\X,\Borel)$ be a measurable space. Recall that a probability {measure} on $\X$, henceforth called a \emph{measure}, is a function $\mu:\Borel\to[0,1]$ satisfying:
\begin{itemize}
\setlength\itemsep{1pt}
	\item[(\textit{i})] $\mu(\emptyset)=0$ and $\mu(\X)=1$;
	\item[(\textit{ii})] if $A,B\in\Borel$ and $A\subseteq B$ then $\mu(A)\leq\mu(B)$;
	\item[(\textit{iii})] if $\{A_i\}_{i=1}^\infty\subseteq\Borel$ are pairwise disjoint then 	$\mu\left(\bigcup_{i=1}^\infty A_i \right) = \sum_{i=1}^\infty \mu(A_i)$.	
\end{itemize}
A measure $\mu$ is \emph{non-trivial} if it vanishes on singletons: \mbox{$\mu(\{x\})=0,\forall x\in\X$}.
Non-trivial measures play a key role in establishing the impossibility of \UBC\ in non-\essep\ spaces and we will be concerned mainly with such measures.

Another important property of a measure is its additivity.
For a cardinal $\kappa$, a measure $\mu$ is \emph{$\kappa$-additive} if for any  $\beta<\kappa$ and any pairwise disjoint measurable family $\{A_\alpha\in \Borel: \alpha<\beta\}$,
\begin{equation}
\label{eq:additivity}
\mu\Big(\bigcup_{\alpha<\beta}A_\alpha\Big) 
=
\sum_{\alpha<\beta} \mu\left( A_\alpha\right)
\gets \sup_{B\in[\beta]^{<\omega}} \sum_{\alpha\in B} \mu(A_{\alpha}).
\end{equation}
By definition, any measure is $\aleph_1$-additive, commonly known as \emph{$\sigma$-additive}.
The following lemma states the main property of non-trivial and $\kappa$-additive measures that will be used here. 
Its proof follows directly from the definitions.
\begin{lemma}
Let $\mu$ be a non-trivial and $\kappa$-additive measure on $\Borel$.
Then any set $A\in\Borel$ with $|A| < \kappa$ has $\mu(A) = 0$.
\label{lem:measure_card}
\end{lemma}

\paragraph{Non-\essep\ metric spaces and real-valued measurable cardinals}
Before giving the formal definition of real-valued measurable cardinals and establishing their relation to general non-\essep\ metric spaces, let us first illustrate the main ideas which will be presented below, using a simple example of an uncountable discrete metric space.
Consider the metric space $([0,1],\dismet)$, where $\dismet$ is the \emph{discrete metric}, defined as
\begin{equation}
\label{eq:discrete_metric}
\dismet(x,x')\gets\pred{x\neq x'}
,\qquad  x,x'\in\X.
\end{equation}%
The Borel $\sigma$-algebra on $([0,1],\dismet)$ is all of $2^{[0,1]}$;
thus, all subsets  of $[0,1]$ are measurable. 
This metric space is clearly non-separable;
the interesting question is whether it is \essep.
In other words, does there exist a measure on $([0,1],\dismet)$ that does not have a separable support?

Note that any non-trivial measure on $([0,1],\dismet)$ suffices to prove that it is non-ES.
Indeed, for any such measure, \lemref{measure_card}, together with the fact that all measures are $\sigma$-additive, implies that any set of positive measure must have an uncountable cardinality.
But any such set is clearly non-separable, due to the discrete nature of the metric space.
Therefore, if a non-trivial measure exists on $([0,1],\dismet)$, then it is non-\essep.
Conversely, if there are no non-trivial measures on $([0,1],2^{[0,1]})$, then the discrete metric space admits only trivial measures with countable support.
So in this case $([0,1],\dismet)$ is \essep.
Thus, the question of whether $([0,1],\dismet)$ is non-\essep\ is equivalent to the question of whether a non-trivial measure exists on this space. 
It is known that the Lebesgue measure, defined on the Borel $\sigma$-algebra generated by open sets on $([0,1], \abs{\cdot})$, cannot be extended to the measurable space $([0,1],2^{[0,1]})$ while simultaneously being translation invariant \cite{fremlin2000measure}. 
However, currently, other non-trivial measures on $([0,1], 2^{[0,1]})$ are not ruled out in ZFC.

More generally, given a cardinal $\kappa$, let $\X$ be some set of that cardinality, and consider the measurable space $\MSk\gets(\X,2^\X)$.
As above, such a space is induced, for example, by the discrete metric $\dismet$. Moreover, whether $(\X, \dismet)$ is \essep\ depends only on the cardinality $\kappa$, and is closely related to the existence of non-trivial measures on $\MSk$, similarly to the example of $([0,1],\dismet)$ above. To characterize the cardinalities for which $\MSk$ is \essep, we use the known concept of \emph{real-valued measurable cardinals} (\RVM).

\begin{definition}
\label{def:rvmc}
A cardinal $\kappa$ is \emph{real-valued measurable} if there exists a non-trivial and $\kappa$-additive measure on $\MSk$.
Any such measure is called a \emph{witnessing measure} for $\MSk$.
\end{definition}

Clearly, any \RVM\ must be uncountable.
We denote by $\lrvmc$ the smallest \RVM; this cardinal exists if some \RVM\ exists, by the well-ordering of the cardinals. 
The following theorem from \cite{billingsley68} characterizes \essep\ metric spaces in terms of $\lrvmc$.
Recall that a set $\discset\subseteq\X$ is \emph{discrete} if for any $x\in \discset$ there exists some $r_x>0$ such that  $B_{r_x}(x)\cap \discset=\set{x}$.

\begin{theorem}[{\cite[Appendix III, Theorem 2]{billingsley68}}]
\label{thm:billignsley}
Let $\lrvmc$ be the smallest real-valued measurable cardinal (if one exists).
Then a metric space $(\X,\rho)$ is \essep\ if and only if every discrete $\discset\subseteq\X$ has $|\discset|<\lrvmc$.
\end{theorem}

\begin{remark}
\label{rem:billingsley_extension}
Theorem \ref{thm:billignsley} is stated in \cite{billingsley68} only for the case $\lrvmc\leq\cont$.
However, one can readily verify that the proof extends essentially verbatim (by replacing ``atomless'' with ``non-trivial'') to the case $\lrvmc>\cont$ as well.
\eor\end{remark}

It follows that whether a given metric space $(\X,\rho)$ is \essep\ or not depends on whether any \RVM\ exists and, if one exists, on the cardinality of the smallest such cardinal, $\lrvmc$.
Assuming that ZFC is consistent (which cannot be proved in ZFC, by G\"{o}del's second incompleteness theorem), it is well known that one cannot prove in ZFC the existence of real-valued measurable cardinals (\RVM).
While it is possible that one can prove in ZFC that \RVM s do not exist, no such proof has been discovered yet. However, quoting Fremlin \citet{fremlin1993real}, ``at present, almost no-one is seriously searching for a proof in ZFC that real-valued measurable cardinals don't exist.''
In fact, currently, the vast majority of set-theoreticians believe that the existence of \RVM\ is \emph{independent} of ZFC, that is, assuming that ZFC is consistent, the existence of an \RVM\ can neither be proven nor disproven from the axioms of ZFC.

In particular, if one adds to ZFC the axiom that no \RVM\ exists, then \emph{any} metric space is \essep.
Alternatively, under some additional properties that are beyond the scope of this paper, one can take $\lrvmc$ to be of cardinality that is arbitrarily large.
For more details see \cite[\S12]{jech}.

The above relations (and their connection to the results to follow) are further discussed in Section~\ref{SEC:disc} and illustrated in Figure~\ref{FIG:classes_diagram} therein.

\begin{remark}
\label{rem:CH_ES}
It is worth mentioning that if one adds to ZFC the Continuum Hypothesis, $\cont=\aleph_1$, which is well known to be independent of ZFC, then if a \RVM\ exists, then it must hold that $\lrvmc>\cont$ \cite{jech}.
In this case, all metric spaces of cardinality $\leq \cont$ are \essep.
In particular, the metric space $([0,1],\dismet)$ discussed at the beginning of this section admits only trivial measures with a countable support.
\end{remark}

\subsection{UBC is impossible in non-\essep\ metric spaces}
\label{SEC:UBC_IMPOSSIBLE}
In this section we prove the following theorem, which readily implies \thmref{equivalence}.

\begin{theorem}
\label{thm:no-ubc}
Let $(\X,\rho)$ be a non-\essep\ metric space and let $\Alg$ be any (possibly random) learning algorithm mapping samples $S\in(\X\times\set{0,1})^{<\omega}$ to classifiers $\Alg(S)\in\{0,1\}^\X$.
Then, there exist a measure $\bmu$ on $\X\times\set{0,1}$ (w.r.t.\ the Borel sets induced by $\rho$), a measurable classifier $\hst:\X\to\set{0,1}$, and an $\eps > 0$ such that, for $n\in\N$ and $S_n \sim \bmu^n$,
\beqn
\label{eq:lb_ineq}
\limsup_{n\to\infty}\E[\err_\bmu(\Alg(S_n))] \geq \err_\bmu(h^*) + \eps
= R_{\bmu}^* + \eps,
\eeqn
where $R_{\bmu}^*$ is the optimal Bayes error.
In particular, no weak or strong UBC algorithm exists for $(\X,\rho)$.
\end{theorem}

For notational simplicity, in the following we denote $\hat h_S := \Alg(S)$ (not to be confused with the 1-NN classifier $h_S$ which we used in previous sections).

\begin{remark}[Measurability of \Alg]
\label{rem:meas-alg}
To be strictly clear about definitions here, note that we require that the learning algorithm be \emph{measurable}, in the sense that for every $\bmu$ and $n$, for $S \sim \bmu^n$, $\hat{h}_{S}$ is a ${\Borel}(L^{1}(\mu))$-measurable random variable, where $\mu$ is the marginal of $\bmu$ on $\X$, and ${\Borel}(L^{1}(\mu))$ is the Borel $\sigma$-algebra on the set of all measurable functions $\X \to \set{0,1}$, induced by the $L^{1}(\mu)$ pseudo-metric.
This is a basic criterion, without which the expected risk of $\hat{h}_{S}$ is not well-defined (among other pathologies).

For deterministic algorithms $\hat{h}$, to satisfy the above criterion, it suffices that the function $(s,x) \mapsto \hat{h}_{s}(x)$ on $(\X \times \set{0,1})^{n} \times \X$ is a measurable $\set{0,1}$-valued random variable, under the product $\sigma$-algebra on $(\X \times \set{0,1})^{n} \times \X$.
To see this, note that for any such function, for $X \sim \mu$ independent of $S \sim \bmu^n$, for any measurable function $f : \X \to \set{0,1}$, we have that $|\hat{h}_{S}(X) - f(X)|$ is a measurable random variable;
hence the variable $\E\!\left[|\hat{h}_{S}(X)-f(X)| \middle| S\right]$ is well-defined and measurable. Therefore, for any $\eps > 0$, the event that $\E\!\left[|\hat{h}_{S}(X)-f(X)| \middle| S\right] \leq \eps$ is measurable.
Thus, the inverse images of balls in the $L^{1}(\mu)$ pseudo-metric are measurable sets, and since these balls generate ${\Borel}(L^{1}(\mu))$, this implies $\hat{h}_{S}$ is a ${\Borel}(L^{1}(\mu))$-measurable random variable. 

In particular, we note that \newname{} satisfies this measurability criterion, since calculating its prediction $\hat{h}_{s}(x)$ involves only simple operations based on the metric $\rho$ (which are measurable, since by definition, $\rho$ induces the topology generating the Borel $\sigma$-algebra), and other basic measurability-preserving operations such as $\argmin$ for a finite number of indices indexing measurable quantities.
Thus, our requirements of $\hat{h}_{S}$ in Theorem~\ref{thm:no-ubc} are satisfied by \newname{}.
\eor\end{remark}

\begin{remark}
\label{rem:cerou_example}
In \cite[Section 2.1]{cerou2006nearest}, the authors define the metric space $(\X,\rho)$, where $\X=[0,1]$ and
\beq
\rho(x,x')=\pred{x\neq x'}\cdot(1+\pred{xx'\neq0})
\eeq
and endow it with the distribution $\mu$, which places a mass of $1/2$ on $x=0$ and spreads the rest of the mass ``uniformly'' on $(0,1]$. The deterministic labeling $h^*(x)=\pred{x>0}$ is imposed. The authors observe that the optimal Bayes risk is $R^*=0$ while the (classical) $1$-NN classifier achieves an asymptotic expected risk of $1/2$ --- in contradistinction to the standard result that
in finite-dimensional spaces $1$-NN is \Bcstn\ in the realizable case. The authors then use this example to argue that ``[separability] is required even in finite dimension''.
We find the example somewhat incomplete, because care is not taken to ensure that $(\X,\rho,\mu)$ is a {\em metric probability space} --- that is, that the $\sigma$-algebra supporting $\mu$ is generated by the open sets of $\rho$.
Indeed, the Borel $\sigma$-algebra generated by $\rho$ is the discrete one, $\Borel=2^{[0,1]}$.
Endowing the latter with a ``uniform'' measure implicitly assumes that the Lebesgue measure on the \emph{standard} Borel $\sigma$-algebra can be extended to all subsets of $[0,1]$ ---  a statement known to be equivalent to $\cont$  being larger than or equal to a real-valued measurable cardinal \cite{jech}.
So the above metric probability space is assumed to be non-\essep, as in \thmref{no-ubc}.
Another objection is that, under any reasonable notion of {\em dimension}, the metric space $(\X,\rho)$ would be considered $\cont$-dimensional rather than finite-dimensional.

It is worth mentioning that by \remref{CH_ES}, if one accepts, say, the Continuum Hypothesis, then the above metric space becomes \essep\ and admits only trivial measures with a countable support (so the standard $k$-NN, and many other algorithms, are in fact \UBC\ in this space). 
\eor\end{remark}  

To prove Theorem~\ref{thm:no-ubc}, we first note that \eqref{lb_ineq} indeed implies that no weak or strong UBC algorithm exists for $(\X,\rho)$ by an application of \cite[Theorem 5.4]{billingsley68}.
To establish \eqref{lb_ineq}, note that since $(\X,\rho)$ is non-\essep, \thmref{billignsley} implies that there exists a discrete set $\discset\subseteq \X$ with $|\discset|=\lrvmc$, where $\lrvmc$ is the smallest real-valued measurable cardinal (see \secref{RVMC_PRE}).
Let $\MSD\gets(\discset,2^\discset)$. By \lemref{borel_subset} in the supplementary material \cite{HKSW_sup}, $2^\discset \subseteq \Borel$.
Hence, it suffices to construct the required adversarial measure on $\discset\times\{0,1\}$.
That being the case, from now on we set without loss of generality $\X\gets\discset$ and $\Borel \gets 2^\discset$.

Below, we split the argument for the construction of the required adversarial measure on $\X\times\{0,1\}$ into two cases:
$$
\text{Case (I): } \lrvmc\leq\cont
\qquad \text{and} \qquad  \text{Case (II): } \lrvmc>\cont.$$
This is manifested by what is known as Ulam's dichotomy. This dichotomy dictates the nature of non-trivial measures in the two cases.
To formally state the dichotomy we first need some additional definitions.

Let $\mu$ be a measure on $\X$. A set $A\subseteq\X$ is an \emph{atom} of $\mu$ if $\mu(A)>0$ and for every measurable $B\subseteq A$ either $\mu(B)=0$ or $\mu(B)=\mu(A)$.
A measure $\mu$ is \emph{atomless} if it has no atoms.
So in an atomless measure, for any $A\in\Borel$ with $\mu(A)>0$ there exists a $B\subset A$ with $0<\mu(B)<\mu(A)$.
Conversely, $\mu$ is \emph{purely atomic} if every $A\in\Borel$ with $\mu(A)>0$ contains an atom.

Clearly, in a countable space all measures are trivial and purely atomic.
However, in uncountable spaces matters are more subtle.
While any atomless measure is non-trivial, one might expect that conversely a non-trivial measure cannot contain an atom. 
However, this is not necessarily the case.
In particular, when $\lrvmc>\cont$, measures on $\X$ that are simultaneously non-trivial and purely atomic exist.

Formally, let $\kappa$ be an \RVM. Recall that a witnessing measure for $\MSk$ is a non-trivial and $\kappa$-additive measure on $\MSk$, namely, a measure defined over all subsets of $\X$ and that vanishes on any set of cardinality $<\kappa$.
We say that $\kappa$ is \emph{two-valued measurable} if there is a $\{0,1\}$-valued witnessing measure on $\MSk$, where a measure is $\{0,1\}$-valued (or \emph{two-valued}) if $\mu(A)\in\{0,1\}$ for all $A\in\Borel$.
Clearly, a two-valued measure is purely atomic and satisfies that, for any countable partition $\{P_i\}_{i\in\N}\subseteq\Borel$ of $\X$, there exists {one and only one} $j\in\N$ such that $\mu(P_j)=1$.
We say that $\kappa$ is \emph{atomlessly measurable} if there is an atomless witnessing measure on $\MSk$.
In 1930, Ulam established the following dichotomy (see \cite[\S543]{fremlin2000measure}).
\begin{theorem}[Ulam's Dichotomy \cite{ulam1930masstheorie}]
\label{thm:ulam_dichotomy}
Let $\kappa$ be a real-valued measurable cardinal.
Then
\begin{itemize}
	\item[(i)] if $\kappa\leq\cont$ then $\kappa$ is atomlessly measurable and every witnessing measure on $\MSk$ is atomless;
	\item[(ii)] if $\kappa>\cont$ then	$\kappa$ is two-valued measurable and every witnessing measure on $\MSk$ is purely atomic.
\end{itemize}
In other words, if $\kappa$ is atomlessly measurable then $\kappa\leq\cont$, while if $\kappa$ is two-valued measurable then $\kappa>\cont$.
\end{theorem}
We now proceed to construct the adversarial measures on $\X\times\{0,1\}$ by considering the two cases ({I}) and ({II}) above separately.

\paragraph{({I}) The case $\lrvmc \leq \cont$}
By Ulam's dichotomy in \thmref{ulam_dichotomy}, $|\X|=\lrvmc$ is atomlessly measurable, so there exists an atomless witnessing measure $\mu$ on $\Borel$.
Fix such a $\mu$ and define the induced set-difference pseudometric 
\beq
\Delta(A,B) = \mu( \{A\cup B\}\setminus \{A \cap B\}),
\qquad A,B\in\Borel.
\eeq
Define the metric space $(\Uorel,\Delta)$, where $\Uorel \subseteq\Borel$ is the quotient $\sigma$-algebra  under the equivalence relation $A\sim B \Leftrightarrow \Delta(A,B)=0$.
The measure $\mu$ induces the corresponding functional $\tilde\mu:\Uorel \to [0,1]$ which agrees with $\mu$ on the equivalence classes.
The following is proved in \appref{nonsep_I} of the supplementary material \cite{HKSW_sup} by an application of Gitik-Shelah Theorem \cite{gitik1989forcings}.

\begin{lemma}
\label{lem:nonsep-eps-sep}
Let $\X$ be a set of an atomlessly-measurable cardinality $\kappa$ and let $\mu$ be a witnessing measure on $\MSk$.
Let $(\Uorel,\Delta)$ be as above.
Then there exist $\eps>0$ and $\H_{\eps} \subseteq \Uorel$ of cardinality $|\H_{\eps} |= \kappa$ that is $\eps$-separated:
\beq
\Delta(U,V)\geq\eps
,\qquad\forall U,V \in \H_{\eps} ,\; U\neq V.
\eeq
\end{lemma}

By Lemma~\ref{lem:nonsep-eps-sep}, there exist an $\eps > 0$ and a set $\H_{\eps} \subseteq \set{0,1}^{\X}$ such that $\forall g,h \in \H_{\eps}$ with $g \neq h$, $\mu(\{ x : g(x) \neq h(x) \}) \geq \eps$, and furthermore $\H_{\eps}$ has cardinality $\lrvmc$:
that is, the same cardinality as $\X$.
Since $\lrvmc$ is atomlessly-measurable, there exists an atomless witnessing measure $\pi$ on $(\H_{\eps},2^{\H_{\eps}})$.
We will construct the distribution $\bmu$ using a random construction, by fixing the marginal $\mu$  on $\X$ and setting $\bmu$ to agree with the classifier $\hst$, which is  $\pi$-distributed, independently of the input to the algorithm. This process is described formally below.

First, we introduce a relaxed objective for the learning algorithm $\Alg$.
Recall that given a labeled sample $S_n\in(\X\times\{0,1\})^{n}$, $\Alg$ outputs a classifier $\hh_{S_n}\in\{0,1\}^{\X}$.
For any sequence $S'_x := \{x'_1,x'_2,\ldots \}\in \X$, $n \in \N$, and $\mathbf{y}' := (y_{1}^{\prime},\ldots,y_{n}^{\prime}) \in \set{0,1}^n$, denote $\hat{h}_{S'_{x},\mathbf{y}'} := \hat{h}_{\{(x_1',y_1^{\prime}),\ldots,(x_n',y_n^{\prime})\}}$ and let $H_{S'_{x}} := \{ \hat{h}_{S'_{x},\mathbf{y}'} : n \in \N, \mathbf{y}' \in \set{0,1}^n \}$. 
This set may be random if the learning algorithm is randomized.
Then note that, for any fixed $\bmu$, denoting by $S := \{(x_1,y_1),(x_2,y_2),\ldots\}$ a countably-infinite sequence
of independent $\bmu$-distributed random variables, and further denoting $S_n := \{(x_1,y_1),\ldots,(x_n,y_n)\}$ and $S_x := \{x_1,x_2,\ldots\}$, we have
\begin{equation*}
\inf_{n} \E\left[ \err_{\bmu}(\hat{h}_{S_n}) \right] 
\geq \E\!\left[ \inf_{h \in H_{S_{x}}} \err_{\bmu}(h) \right].
\end{equation*}

Now take $\bS_{x} = \{x_1,x_2,\ldots\}$ to be an i.i.d.~$\mu$-distributed sequence, and let $\hst \sim \pi$ independently of $\bS_{x}$.
Let $\bmu$ have marginal $\mu$ over $\X$ and define $\bmu$ such that \mbox{$\bmu( \{ (x,\hst(x)) : x \in \X \} ) = 1$;} that is, $\bmu$ is an $\hst$-dependent random measure.
Note that $\err_{\bmu}(\hst) = 0$ (a.s.), and hence also that any $h$ has $\err_{\bmu}(h) = \mu(\{ x^{\prime} : h(x^{\prime}) \neq \hst(x^{\prime}) \})$ (a.s.).
Furthermore, by the assumed measurability of the learning algorithm, for each $\mathbf{y}$ we have that $\hat{h}_{\bS_{x},\mathbf{y}}$ is a ${\Borel}(L^{1}(\mu))$-measurable random variable, and $\hst$ is also ${\Borel}(L^{1}(\mu))$-measurable (its distribution is $\pi$, which is defined on this $\sigma$-algebra).
Therefore, $\mu( \{ x^{\prime} : \hat{h}_{\bS_{x},\mathbf{y}}(x^{\prime}) \neq \hst(x^{\prime}) \} )$ is a measurable random variable, equal (a.s.) to $\err_{\bmu}(\hat{h}_{\bS_{x},\mathbf{y}})$.

In particular, this implies that $\E\!\left[ \inf\limits_{h \in H_{\bS_{x}}} \err_{\bmu}(h) \right]$ is well-defined, 
and by the law of total expectation, 
\begin{align*}
& \E\!\left[ \inf_{h \in H_{\bS_{x}}} \err_{\bmu}(h) \right] 
= \E\!\left[ \E\!\left[ \inf_{h \in H_{\bS_{x}}} \err_{\bmu}(h) \middle| H_{\bS_{x}} \right] \right]
\\ & \geq \E\!\left[ (\eps/2) \P\!\left( \inf_{h \in H_{\bS_{x}}} \err_{\bmu}(h) > \eps/2  \middle| H_{\bS_{x}} \right) \right]
\\ & = \E\!\left[ (\eps/2) \pi\!\left( h^{\prime} \in \H_{\eps} : 
\inf_{h \in H_{\bS_{x}}} \mu( \{ x^{\prime} : h(x^{\prime}) \neq h^{\prime}(x^{\prime}) \} ) > \eps/2 
\right) \right].
\end{align*}
Then note that each element of $H_{\bS_{x}}$ can be $(\eps/2)$-close to at most one element of $\H_{\eps}$, and since $H_{\bS_{x}}$ is a countable set, this implies that, given $H_{\bS_{x}}$, the set $H_{\bS_{x}}^{\eps} = \{ h^{\prime} \in \H_{\eps} : \inf\limits_{h \in H_{\bS_{x}}} \mu(\{ x^{\prime} : h(x^{\prime}) \neq h^{\prime}(x^{\prime}) \}) \leq \eps/2 \}$ is countable.

But since $\pi$ vanishes on singletons, we have $\pi(H_{\bS_{x}}^{\eps}) = 0$.
Thus, given $H_{\bS_{x}}$, 
\begin{equation*}
\pi\!\left( h^{\prime} \in \H_{\eps} : 
\inf_{h \in H_{\bS_{x}}} \mu( \{ x^{\prime} : h(x^{\prime}) \neq h^{\prime}(x^{\prime}) \} ) > \eps/2 \right) = 1, 
\end{equation*}
so that altogether we have 
\begin{equation*}
\E\!\left[ \inf_{h \in H_{\bS_{x}}} \err_{\bmu}(h) \right] 
\geq \eps/2.
\end{equation*}
In particular, this also implies there exist fixed choices of $h^*$ for which \eqref{lb_ineq} holds.
This completes the proof for the case (\textit{I}).

\paragraph{({II}) The case $\lrvmc > \cont$}
By Ulam's dichotomy in \thmref{ulam_dichotomy}, $|\X|=\lrvmc$ is two-valued measurable, so there exists a two-valued witnessing measure $\mu$ on $\Borel$.
As the following lemma shows, $\mu$ can be taken to further satisfy a key homogeneity property.
The lemma is proved in \appref{nonsep_II} of the supplementary material \cite{HKSW_sup}, where it is shown to follow by combining Theorems 10.20, 10.22 in \cite{jech} and Ulam's Theorem \ref{thm:ulam_dichotomy}.

\begin{lemma} 
\label{lem:homogeneous}
Let $\X$ be of a two-valued measurable cardinality $\kappa$ and let $\MSk=(\X,2^\X)$.
Then, there is a witnessing measure $\mu$ on $\MSk$   such that for any function $f:[\X]^{<\omega} \to \R$, there exists a $U\subseteq \X$ with $\mu(U)=1$ such that $U$ is \emph{homogeneous} for $f$, that is, for every $n\in\N$, there exists a $C_n\in\R$ such that $f(\vectosetsym) = C_n$ for all $\vectosetsym\in[U]^n$.
\end{lemma}

Let $\mu$ be a two-valued witnessing measure on $\Borel=2^\X$ as furnished by Lemma \ref{lem:homogeneous}.
For a label $y\in\{0,1\}$ and any two-valued witnessing measure $\phi$, let $\bphi_y$ be the measure over $\X\times\{0,1\}$ with $\phi$ as its marginal over $\X$ and 
\beq
\bphi_y(Y=y \gn X=x) = 1,
\qquad \forall x\in\X.
\eeq
For $\mu$ as above, and any other two-valued witnessing measure $\phi$, define
\beqn
\label{eq:mixture}
\mix_{\phi} \gets  \frac{2}{3}\bphi_1 + \frac{1}{3}\bmu_0.
\eeqn
We will show that there exists a two-valued witnessing measure $\nu\gets \nu(\mu,\Alg)$ $\neq\mu$ such that $\Alg$ cannot be \Bcst\ on both $\mix_{\mu}$ and $\mix_{\nu}$.
To this end, we will use the following properties of the mixture $\mix_{\phi}$, proved in \appref{nonsep_II} of the supplementary material \cite{HKSW_sup}.
\begin{lemma}
\label{lem:bayes_optimal}
Let $\nu\neq\mu$ be any two distinct two-valued measures on $(\X,\Borel)$ and let $\mix_\phi$ with $\phi\in\{\mu,\nu\}$ be as in \eqref{mixture}.
\begin{itemize}
\item[(i)] Any Bayes-optimal classifier $\hst$ on $\mix_\mu$ achieves the optimal Bayes-error $\err_{\mix_{\mu}}(\hst) = \frac{1}{3}$ if and only if $\E_{X\sim \mu} [\hst(X)] = 1$.
\item[(ii)] Any Bayes-optimal classifier $\hst$ on $\mix_\nu$ achieves the optimal Bayes-error $\err_{\mix_{\nu}}(\hst) = 0$ if and only if $\E_{X\sim \mu} [\hst(X)]=0$ and $\E_{X\sim \nu} [\hst(X)]= 1$.
\end{itemize}
\end{lemma}

Let $\Alg:(\X\times\{0,1\})^{<\omega} \to 2^\X$ be any (possibly randomized) learning algorithm, and recall that $\hh_{S}$ denotes the classifier output for data set $S$; for $S$ and $X$ independent samples from Borel measures on $\X$, we suppose that $\hh_{S}(X)$ is a measurable random variable (by definition of learning algorithm; see Remark \ref{rem:meas-alg}).
Let $\nu\neq\mu$ be a two-valued witnessing measure to be chosen below.
Consider the quantity
\beq
\muhlabel_n^{\phi} \gets 
\E_{S_n\sim (\mix_{\phi})^n} \left[ 
\E_{X\sim\mu}\left[\hh_{S_n}(X)\right] 
\right]
,\qquad \phi\in\{\mu,\nu\}.
\eeq
In the case of a randomized $\Alg$, also add an innermost expectation over the independent randomness of $\Alg$ in the above expression.
By Lemma \ref{lem:bayes_optimal}, for $\Alg$ to be \Bcstn\ on both $\mix_{\nu}$ and $\mix_{\mu}$ we must have
\beqn
\label{eq:R_munu_lim}
\muhlabel_n^{\phi} 
\xrightarrow[n\to\infty]{} 
\delta_{\mu,\phi}
,\qquad \phi\in\{\mu,\nu\},
\eeqn
where $\delta_{\mu,\phi}$ is the Kronecker delta.
So to prove the claim it suffices to show that we can choose $\nu \gets \nu(\mu,\Alg)$ such that \eqref{R_munu_lim} does not hold.

Given a labeled sample $S_n =(\bX_n,\bY_n) \sim (\mix_{\phi})^n$ with $\phi\in\{\mu,\nu\}$, let
$n_1 \gets n_1(\bY_n)=\sum_{i=1}^n Y_i$
and
$n_0 \gets  n-n_1$
be the random number of samples in $S_n$ with labels $1$ and $0$ respectively, and let $\bX_n^0\in\X^{n_0}$ and $\bX_n^1\in\X^{n_1}$ be the corresponding instances in $\bX_n$.
For notational simplicity we write $\bX_n = (\bX_n^0,\bX_n^1)$ where it is understood that the embedding of $\bX_n^0$ and $\bX_n^1$ in $\bX_n$ is in accordance with $\bY_n$.
Note that $\bY_n\sim (\text{Bernoulli}\!\left(\frac{2}{3})\right)^n$ irrespectively of $\mu$ and $\phi$.
In addition, given $\bY_n$ we have that $\bX_n^0$ and $\bX_n^1$ are independent and $\bX_n^0|\bY_n \sim \mu^{n_0}$ and $\bX_n^1|\bY_n \sim \phi^{n_1}$.
We decompose
\beqn
\nonumber
\muhlabel_n^{\phi} &=& \E_{S_n\sim (\mix_{\phi})^n} \left[ \E_{X\sim\mu}\left[\hh_{S_n}(X)\right] \right]
\\
\nonumber
&=&
\E_{\bY_n}\;
\E_{\bX_n | \bY_n} \left[ \E_{X\sim\mu}\left[\hh_{(\bX_n,\bY_n)}(X)\right]
 \right]
\\
\label{eq:L_n_phi}
& = &
\E_{\bY_n}
\E_{\bX_n^1\sim\phi^{n_1}}
\E_{\bX_n^0\sim\mu^{n_0}}
 \left[ \E_{X\sim\mu}\left[\hh_{((\bX_n^0,\bX_n^1),\bY_n)}(X)\right]
 \right].
\eeqn

Towards applying Lemma \ref{lem:homogeneous}, we first need to translate our reasoning about a random vector $\bX= (X_1,\dots,X_k)\sim\phi^k$ with $k\in\N$ into reasoning about the random set of its distinct elements, $\vectoset{\bX} \gets \bigcup_{i=1}^{k} \{X_i\}$.
Since $\phi$ vanishes on singletons, all instances in $\bX$ are distinct with probability one,
\beqn
\label{eq:|D|=k}
\P_{\bX\sim\phi^k}\left[|\vectoset{\bX}| = k\right] = 1.
\eeqn
Fixing an ordering on $\X$, for any finite set $\vectosetsym=\{w_1,\dots,w_k\}\in[\X]^k$, denote by $\Pi(\vectosetsym)$ the distribution over vectors $\bX' = (w_{\pi(1)},\dots,w_{\pi(k)}) \in \vectosetsym^{k}$ as induced by a random permutation $\pi$ of the instances in $\vectosetsym$.
Then, by \eqref{|D|=k} and the fact that $\phi^k$ is a product measure, we have that for any measurable function $f:\X^{k}\to[0,1]$ the following symmetrization holds,
\beqn
\label{eq:gen_f}
\E_{\bX\sim\phi^{k}}[f(\bX)] = \E_{\bX\sim\phi^{k}}
\left[
\E_{\bX'\sim\Pi(\vectoset{\bX})}[f(\bX')]
\;\middle|\;
|\vectoset{\bX}|=k
\right].
\eeqn
For every $\bY_n\in \{0,1\}^n$ define $F_{\bY_n}:[\X]^{n_1} \to \R$ by
\beq
F_{\bY_n}(\vectosetsym) = 
\E_{\bX^1\sim\Pi(\vectosetsym)}
\E_{\bX_n^0\sim\mu^{n_0}}\left[\E_{X\sim\mu}[\hh_{((\bX_n^0,\bX^1),\bY_n)}(X)]
\right]
,\qquad \vectosetsym\in[\X]^{n_1}.
\eeq
In the case of randomized $\Alg$, we also include an innermost conditional expectation over the value of $\hh_{((\bX_n^0,\bX^1),\bY_n)}(X)$ given $\bX_n^0,\bX^1,\bY_n,X$.
Putting this in \eqref{L_n_phi} while using \eqref{|D|=k} and \eqref{gen_f},
\beq
\muhlabel_n^{\phi} =
\E_{\bY_n}
\E_{\bX\sim\phi^{n_1}}
\left[ 
 F_{\bY_n}(\vectoset{\bX})
 \;\middle|\; |\vectoset{\bX}|=n_1\right].
\eeq
By the choice of $\mu$, Lemma \ref{lem:homogeneous} implies there exist $C_{\bY_n}\in\R$ and $U_{\bY_n}\subseteq \X$ with $\mu(U_{\bY_n})=1$ such that $U_{\bY_n}$  is {homogeneous} for $F_{\bY_n}$, namely, $F_{\bY_n}(\vectosetsym) = C_{\bY_n},\forall \vectosetsym \in [U_{\bY_n}]^{n_1}$.
Let
\beqn
\label{eq:U}
U = \bigcap_{n\in\N} \bigcap_{\bY_n\in\{0,1\}^n} U_{\bY_n}.
\eeqn
Then $U$ is simultaneously homogeneous for all $\{F_{\bY_n}\}$,
\begin{equation}
\label{eq:homog_F_Y}
F_{\bY_n}(\vectosetsym) = C_{\bY_n},
\qquad \forall n\in\N,\;\; \forall \bY_n\in\{0,1\}^n, \;\; \forall \vectosetsym\in [U]^{n_1}.
\end{equation}
In addition, by Lemma \ref{lem:measure_1_intersect} in the supplementary material \cite{HKSW_sup}, $\mu(U)=1$.

We are now in position to choose $\nu \gets \nu(\mu,\Alg)$.
By Lemma \ref{lem:measure_card}, we may split $U$ in \eqref{U} into two disjoint sets $B$ and $U\setminus B$ such that $|B|=|U\setminus B|=|\X|$.
Since $\mu$ is two-valued, we may assume without loss of generality that $\mu(B)=0$ (so $\mu(U\setminus B)=1$).
Since $|B|$ is a two-valued measurable cardinal, there exists a two-valued witnessing measure $\nu'$ on $(B,2^B)$ with $\nu'(B)=1$.
Extend $\nu'$ to a measure $\nu$ over all $\Borel$ by $\nu(A) = \nu'(A \cap B), \forall A\subseteq\X$.
Then, $\nu \neq \mu$ and $\nu(U)=\mu(U)=1$.
By the last equality, for $\phi\in\{\mu,\nu\}$ and $\forall k\in\N$,
$
\Pr_{\bX \sim \phi^{k}}
\left[
\vectoset{\bX} \in [U]^{k}
\;\middle|\; |\vectoset{\bX}|=k
\right]=1.
$
So, for $\phi\in\{\mu,\nu\}$,
\beqn
\nonumber
\muhlabel_n^{\phi} &=&
\E_{\bY_n}\E_{\bX\sim\phi^{n_1}} \big[ 
 F_{\bY_n}(\vectoset{\bX})
 \;\big|\; |\vectoset{\bX}|=n_1\big]
 \\
 \nonumber
 &=&
 \E_{\bY_n}\E_{\bX\sim\phi^{n_1}} \big[ 
 F_{\bY_n}(\vectoset{\bX})
 \;\big|\; |\vectoset{\bX}|=n_1 \; \wedge \; 
\vectoset{\bX} \in [U]^{n_1} \big]
\\
&=& 
\nonumber
 \E_{\bY_n}\E_{\bX\sim\phi^{n_1}} \big[ 
 \,C_{\bY_n}
 \;\big|\; |\vectoset{\bX}|=n_1 \; \wedge \; 
\vectoset{\bX} \in [U]^{n_1} \big]
 \\
 \nonumber
 \label{eq:E_C_Y_n}
 &=& 
 \E_{\bY_n} \left[C_{\bY_n}\right],
\eeqn
where we used \eqref{homog_F_Y} and the fact that $C_{\bY_n}$ does not depend on $\bX$. 
Since $\E_{\bY_n}[C_{\bY_n}]$ is independent of $\phi$, we conclude that
$\muhlabel_n^{\mu} = \muhlabel_n^{\nu}$ for all $n\in\N$.
However by \eqref{R_munu_lim}, for $\Alg$ to be \Bcstn\ on $\mix_{\mu}$ and $\mix_{\nu}$ we must have
$\muhlabel_n^{\mu}\xrightarrow[n\to\infty]{} 1$ and $\muhlabel_n^{\nu} \xrightarrow[n\to\infty]{} 0$.
Thus $\Alg$ cannot be \Bcstn\ on both $\mix_{\mu}$ and $\mix_{\nu}$.
In particular, \eqref{lb_ineq} holds with $\eps=1/4$.

\section{Discussion}
\label{SEC:disc}
We have exhibited a computationally efficient multiclass learning algorithm, \newname, that is universally strongly \Bcstn\ (UBC) in all essentially separable (ES) metric spaces.
In contrast, we showed that in non-\essep\ spaces, no algorithm can be UBC.
As such, \newname{} is optimistically universal (in the terminology of \cite{DBLP:journals/corr/Hanneke17}) --- it is universally  \Bcstn\ in all metric spaces that admit such a learner.
We note that in this work, we do not study the rates of decay of the excess risk, leaving this challenging open problem for future study.

\begin{figure}%
\includegraphics[width=.3\columnwidth]{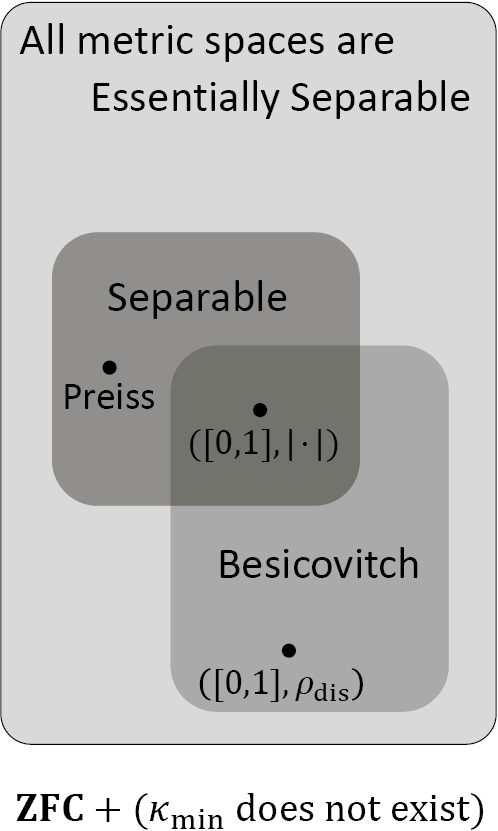}%
\hfill
\includegraphics[width=.3\columnwidth]{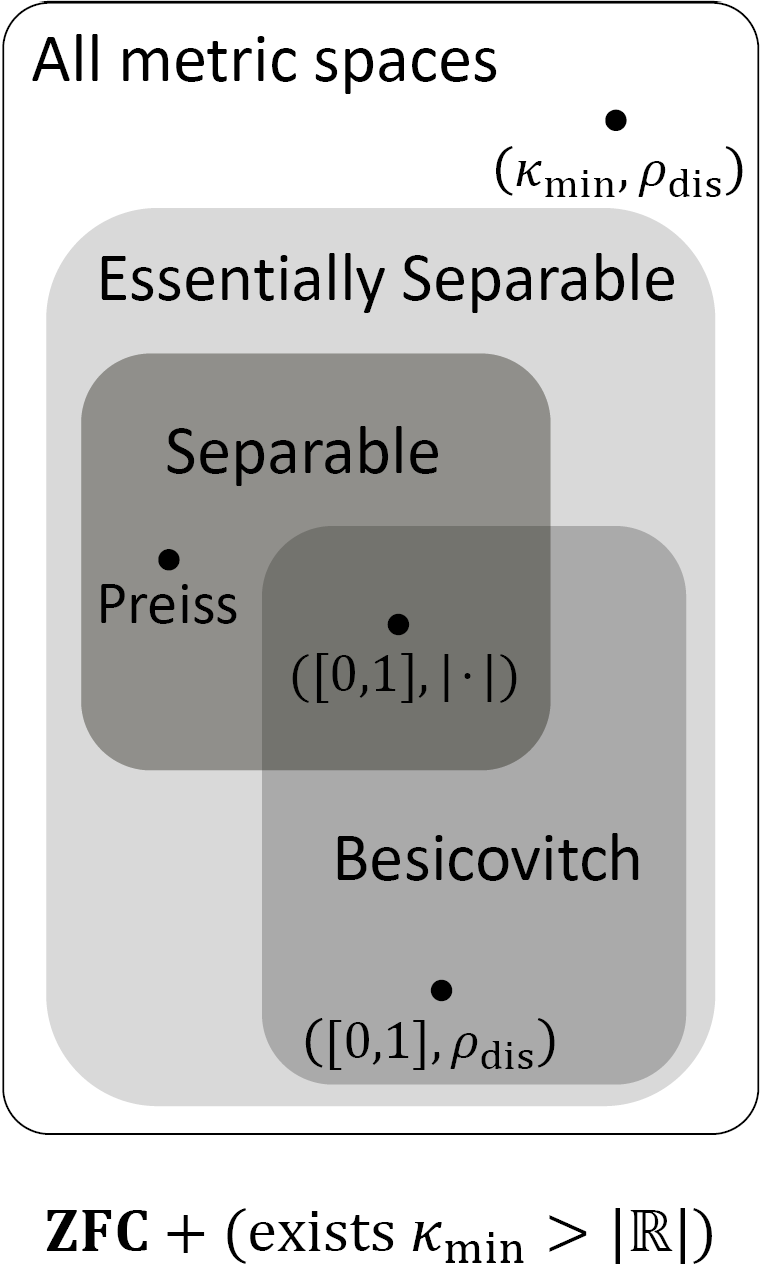}%
\hfill
\includegraphics[width=.3\columnwidth]{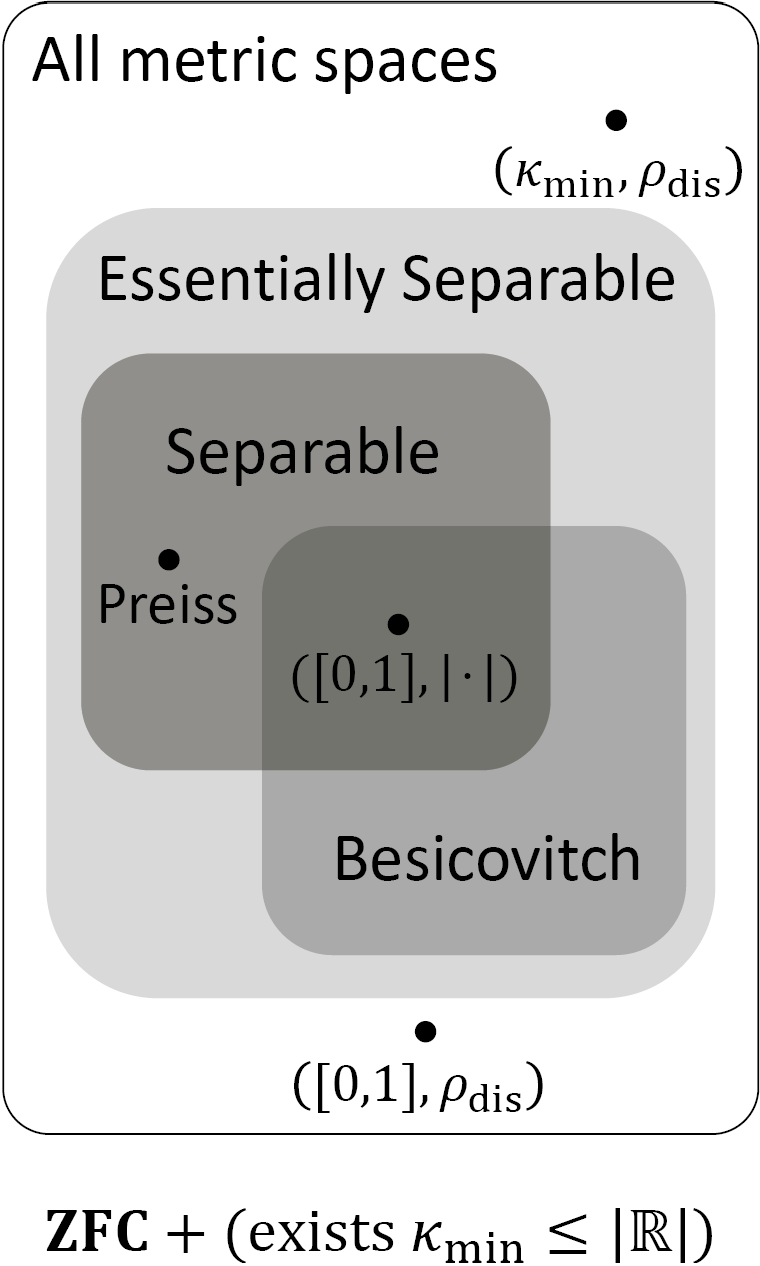}%
\caption{The classes of metric spaces discussed in this paper and their inclusion relationships in the three cases where: no \RVM\ exists (left); minimal \RVM\ is $\lrvmc>\cont$ (middle); and minimal \RVM\ is $\lrvmc\leq\cont$ (right). All three cases are believed to be valid extensions of ZFC.
The metric space $(\lrvmc,\rho_{\text{dis}})$ corresponds to a discrete one of cardinality $\lrvmc$; it is not \essep\ but any discrete metric space of cardinality $<\lrvmc$ is \essep.
The shaded area named ``Besicovitch'' and the specific metric space ``Preiss'' are as discussed in the text.}%
\label{FIG:classes_diagram}%
\end{figure}

By definition, any separable metric space is \essep.
As discussed in \secref{intro}, consistency of NN-type algorithms in general separable metric spaces was studied in \cite{MR2327897,MR2235289,MR2654492,cerou2006nearest,MR1366756, forzani2012consistent}.
In particular, in \cite{MR2327897,cerou2006nearest,forzani2012consistent}, a characterization of the metric spaces in which an algorithm is universally \Bcstn\ was given for several such algorithms, in terms of Besicovitch-type conditions.
As a notable example, it is shown in \cite{cerou2006nearest} that for any separable metric space $\X$, a sufficient condition for the $k$-NN algorithm (with an appropriate choice of the number of neighbors $k$) to be \Bcstn\ for a distribution $\bmu$ over $\X\times\{0,1\}$ is that for all $\eps>0$,
\begin{equation}
\label{eq:besic_cerou}
\lim_{r\to 0^+}
\P
\left\{ 
\frac{1}{\mu(B_r(X))}  \int_{B_r(X)} |\eta(z) - \eta(X)| \diff\mu(z)
> \eps
\right\}
= 0,
\end{equation}
where $\mu$ is the marginal of $\bmu$ over $\X$ and $\eta(x) \gets\P(Y=1 \gn X=x)$.
It is also shown in \cite{cerou2006nearest} that in the realizable case, where $\eta(x)\in\{0,1\}$ for all $x\in X$, a violation of \eqref{besic_cerou} implies that $k$-NN is inconsistent.
Say that a metric space satisfies the \emph{universal Besicovitch condition} if \eqref{besic_cerou} holds \emph{for all} measures $\bmu$ over the Borel $\sigma$-algebra.
By Besicovitch's density theorem \cite[\S 472]{fremlin2000measure}, the metric space $(\R^d,\nrm{\cdot}_2)$ --- and more generally, any finite-dimensional normed space --- satisfies this condition, so $k$-NN is \UBC\ on such spaces. 
In contrast, in infinite-dimensional separable spaces, such as $\ell_2$, a violation of \eqref{besic_cerou} can occur \cite{preiss1979invalid,MR609946,MR1974687}. One such example is the separable metric probability space studied in \cite{DBLP:conf/nips/KontorovichSW17}, building upon a construction of Preiss \cite{preiss1979invalid}.
While the $k$-NN algorithm is provably not \UBC\ in this space, \newname{} is.
As far as we know, \newname{} is the first algorithm known to be \UBC\ (weakly or strongly) in any separable metric space.

As discussed in \secref{RVMC_PRE}, the essential separability of non-separable metric spaces is believed to depend on set-theoretic axioms that are independent of ZFC, and in particular on the cardinality of the minimal \RVM, $\lrvmc$:
a metric space is non-\essep\ if and only if it contains a discrete subset of cardinality $\lrvmc$.
\figref{classes_diagram} gives a pictorial illustration of the possible relationships between the following types of metric spaces: separable, (uniform) Besicovitch, ES, and all spaces, depending on the set-theoretic model. 
If one adopts a model in which no \RVM\ exist, then any discrete subspace of a metric space admits only trivial, purely-atomic measures. In this case, abbreviated as $\textsc{ZFC}+(\lrvmc \text{ does not exist})$ in the left panel of \figref{classes_diagram}, all metric spaces are \essep, and \newname{} is \UBC\ on any metric space.
Alternatively, if one adopts a set-theoretic model in which an \RVM\ exists, then discrete subspaces of $\X$ of cardinality $\geq \lrvmc$ admit also non-trivial measures. As shown in \secref{UBC_IMPOSSIBLE},  such measures exclude the possibility of a \UBC\ algorithm.
The nature of the non-trivial measures, being atomless or purely atomic, depends on whether $\lrvmc>\cont$ or $\lrvmc\leq\cont$, which are illustrated on the middle and right panels of \figref{classes_diagram} respectively.

Lastly, we note that our argument for the impossibility of \UBC\ in non-\essep\ metric spaces is based solely on the real-valued measurability of the cardinality of discrete subspaces of $\X$. This raises a natural question: Assuming no cardinal is real-valued measurable, are there any topological spaces (which by the results above must be non-metric) in which no \UBC\ algorithm exists?

To summarize, in this work we provided the first multiclass learning algorithm that is universally \Bcstn\ in any metric space where
such an algorithm exists. Moreover, we provided a characterization of these metric spaces. The study of learnability in general spaces is fundamental, and provides many open questions for future research. 


\paragraph{Acknowledgments}
We thank Vladimir Pestov for sharing with us his proof of the existence of a measurable total order.
We also thank Robert Furber, Iosif Pinelis, Menachem Kojman, and Roberto Colomboni for 
helpful discussions.

\newcommand{\BoundAbrv}{B}

\appendix

\section{Auxiliary lemmas for Section \ref{SEC:COMPRESSION_SCHEME}}
\label{ap:proofs}

\subsection{Lipschitz functions are dense in $L^1(\mu)$}
The following denseness result is used in proving \lemref{richness}. We believe this fact to be classical, 
but were unable to locate an appropriate citation, 
so for completeness we include a brief proof.

\begin{lemma}
\label{lem:dense_cont}
For every metric probability space
$(\X,\rho,\mu)$, 
the set of Lipschitz functions $f:\X\to\R$ is dense in 
$L^1(\mu)= \{f: \int {\abs{f}} \diff\mu < \infty\}$.
In other words, for any $\eps>0$ and $f\in L^1(\mu)$, 
there is an $L < \infty$ and an $L$-Lipschitz function $g\in L^1(\mu)$ such that
\(
\int \abs{f-g} \diff\mu < \eps.
\)
\end{lemma}

\begin{proof}
The proof follows closely that of a weaker result from \cite[Section 37, Theorem 2]{kolmogorov:70}.
It relies on the fact that, for any probability measure $\mu$ on a Borel $\sigma$-algebra $\Borel$, $\mu$ is \emph{regular} \cite[Theorem 17.10]{kechris:95}.
In particular, for every $A \in \Borel$, $\mu(A) = \sup\limits_{F \in \mathcal{F} : F \subseteq A} \mu(F)$, where $\mathcal{F}$ is the \emph{closed} sets (under the topology that generates $\Borel$).

For any $A \in \Borel$ and $\eps > 0$, regularity implies that there is an $F \in \mathcal{F}$ with $F \subseteq A$ and $\mu(A \setminus F) < \eps/2$.
Now denote $G_{r} = \bigcup_{x \in F} B_{r}(x)$.
Since $\X \setminus F$ is open, for any $x^{\prime} \notin F$, there is an $r > 0$ with $B_{r}(x^{\prime}) \subseteq \X \setminus F$, and hence $x^{\prime} \notin G_{r}$.  Together with monotonicity of $G_{r}$ in $r$, this implies $G_{r} \setminus F \to \emptyset$  as $r \to 0$.
Thus, by continuity of probability measures, there is an $r > 0$ such that $\mu(G_{r} \setminus F) < \eps/2$.
Furthermore, for this $r$, $G_{r} \supseteq F$ and $G_{r}$ is a union of open sets, hence open.
Thus, denoting $F_{r} = \X \setminus G_{r}$, $F_{r}$ is a closed set, disjoint from $F$, and (by definition of $G_{r}$) satisfies $\inf_{x \in F, x^{\prime} \in F_{r}} \dist(x,x^{\prime}) \geq r > 0$.

Now define  
\begin{equation*}
g_{A,\eps}(x) = \frac{ \inf_{x^{\prime} \in F_{r}} \dist(x^{\prime},x) }{ \inf_{x^{\prime} \in F_{r}} \dist(x^{\prime},x) + \inf_{x^{\prime} \in F} \dist(x^{\prime},x) }.
\end{equation*}
In particular, note that $g_{A,\eps}(x) = 1$ for $x \in F$, $g_{A,\eps}(x) = 0$ for $x \in F_{r}$, and every other $x$ has $g_{A,\eps}(x) \in [0,1]$.  
This implies $\{ x : g_{A,\eps}(x) < 1, \ind_{A}(x) = 1 \} \subseteq A \setminus F$ and $\{ x : g_{A,\eps}(x) > 0, \ind_{A}(x) = 0 \} \subseteq (\X \setminus F_{r}) \setminus A = G_{r} \setminus A \subseteq G_{r} \setminus F$, so that 
\begin{equation*}
\int \abs{\ind_{A} - g_{A,\eps}} \diff\mu 
\leq \mu( A \setminus F ) + \mu( G_{r} \setminus F ) 
< \eps.
\end{equation*}
Furthermore, since $F$ and $F_{r}$ are $r$-separated, $g_{A,\eps}$ is $\frac{1}{r}$-Lipschitz, and since $g_{A,\eps}$ is bounded we also have $g_{A,\eps} \in L^1(\mu)$.
Thus, we have established the desired result for indicator functions.

To extend this to all of $L^1(\mu)$, we use the ``standard machinery'' technique.
By definition of Lebesgue integration, for any $f \in L^1(\mu)$ and $\eps > 0$, 
there exists a finite simple function $f_{\eps}$ with $\int \abs{f-f_{\eps}} \diff\mu < \eps/2$: 
that is, there is an $n \in \nats$, $a_{1},\ldots,a_{n} \in \reals$, and $A_{1},\ldots,A_{n} \in \Borel$ with $f_{\eps}(x) = \sum_{i=1}^{n} a_{i} \ind_{A_{i}}(x)$.
Now let $a^{*} = \max\{\abs{a_{1}},\ldots,\abs{a_{n}},1\}$ and denote $\eps^{\prime} = \eps / (2 n a^{*})$.
By the above, for each $i \in \{1,\ldots,n\}$, there exists a Lipschitz function $g_{A_{i},\eps^{\prime}} \in L^1(\mu)$ with $\int \abs{ \ind_{A_{i}} - g_{A_{i},\eps^{\prime}} } \diff\mu < \eps^{\prime}$.
Therefore, denoting $g = \sum_{i=1}^{n} a_{i} g_{A_{i},\eps^{\prime}}$, we have 
\begin{equation*}
\int \abs{ f_{\eps} - g } \diff\mu 
\leq a^{*} \sum_{i=1}^{n} \abs{ \ind_{A_{i}} - g_{A_{i},\eps^{\prime}} } \diff\mu 
< a^{*} n \eps^{\prime} 
= \eps/2.
\end{equation*}
Together we have that $\int \abs{ f - g } \diff\mu \leq \int \abs{ f - f_{\eps} } \diff\mu + \int \abs{ f_{\eps} - g } \diff\mu < \eps$.
Since a finite linear combination of Lipschitz functions is still Lipschitz, this establishes the claim for all $f \in L^1(\mu)$.
\end{proof}

\subsection{Proof of \lemref{KSUbound}} \label{ap:KSUbound}
First, note that $h_{\tS_n} = h_{S_n(\bm i , \bm j)}$ may be expressed as the value of a reconstruction function $h_{\tS_n} = \Phi(S_n(\bm i),S_n(\bm j))$, where the function $\Phi$ generally takes as arguments two equal-length sequences $S = \{(x_i,y_i)\}_{i=1}^{m/2} \in (\X\times\Y)^{m/2}$ and $S' = \{(x'_{i'},y'_{i'})\}_{i'=1}^{m/2} \in (\X\times\Y)^{m/2}$ (for any even $m \in \nats$).
The first sequence $S$ is used to reconstruct the Voronoi partition $\Vor(\{x_i\}_{i=1}^{m/2}) = \left\{ V_1(\{x_i\}_{i=1}^{m/2}),\ldots,V_{m/2}(\{x_i\}_{i=1}^{m/2}) \right\}$ and the second sequence $S'$ is used to specify the label predicted in each Voronoi cell: 
by construction, each cell $V_j(\{x_i\}_{i=1}^{m/2})$ contains exactly \emph{one} of the $x'_{i'}$ points in $S'$, so $h := \Phi(S,S')$ is defined as the unique function that, for every $j \in [m/2]$, $h(x) = y'_{i'}$ for every $x \in V_j(\{x_i\}_{i=1}^{m/2})$;
for completeness, $h$ may be defined as an arbitrary measurable function in the case that not every 
$V_j(\{x_i\}_{i=1}^{m/2})$ contains exactly one of the $x'_{i'}$ points in $S'$.
We may then note that, for any permutations $\sigma, \sigma' : [m/2] \to [m/2]$, we have 
$$\Phi(\{(x_{\sigma(i)},y_{\sigma(i)})\}_{i=1}^{m/2},\{(x'_{\sigma'(i')},y'_{\sigma'(i')})\}_{i'=1}^{m/2})
= \Phi(\{(x_{i},y_{i})\}_{i=1}^{m/2},\{(x'_{i'},y'_{i'})\}_{i'=1}^{m/2}).$$
Thus, $\Phi$ is invariant to permutations of each of the two sequences.

Proceeding analogously to \cite{graepel2005pac}, we can use the above invariance of $\Phi$ to arrive at a bound $Q$ for $(\alpha,m)$-compressions which satisfies the required properties.
Specifically, for any even $m \in [n-2]$, let $\I_{n,m}$ denote the set of all subsets of $[n]$ of size $m/2$.
For any $I \in \I_{n,m}$, let ${\bm i}(I)$ denote the sequence of elements of $I$ enumerated in increasing order.
For any $I,I' \in \I_{n,m}$, define 
\begin{equation*} 
\hat{R}(I,I';S_{n}) := 
\frac{1}{n - |I \cup I'|} \sum_{i \in [n] \setminus (I \cup I')} \ind[ \Phi(S_n({\bm i}(I)),S_n({\bm i}(I')))(X_i) \neq Y_i ].
\end{equation*}
Note that any $(\alpha,m)$-compression $S'_{n}$ of $S_{n}$ has $h_{S'_{n}} = \Phi(S_n({\bm i}(I)),S_n({\bm i}(I')))$ for some $I,I' \in \I_{n,m}$.
Thus, for any $\delta \in (0,1)$, letting 
\begin{equation*}
\BoundAbrv(a,b) := a + \sqrt{\frac{8a}{b} \ln\!\left(\frac{2n|\I_{n,m}|^2}{\delta}\right)} +  \frac{9}{b} \ln\!\left(\frac{2n|\I_{n,m}|^2}{\delta}\right),
\end{equation*}
we have, for an even $m \leq n-2$
\begin{align*}
& \P\!\left[ |S'_{n}| = m/2 \text{ and } 
\err(h_{S'_{n}}) > \BoundAbrv\!\left(\tfrac{n}{n-m} \serr(h_{S'_{n}}), n-m \right) \right]
\\ & \leq \P\!\left[ \exists I,I' \!\in\! \I_{n,m} : 
\err(\Phi(S_n({\bm i}(I)),S_n({\bm i}(I')))) \!>\! \BoundAbrv\!\left( \hat{R}(I,I';S_{n}), n \!-\! |I\!\cup\! I'| \right)
\right]
\\ & \leq \sum_{I,I' \in \I_{n,m}} \!\P\!\left[ \err(\Phi(S_n({\bm i}(I)),S_n({\bm i}(I')))) \!>\! \BoundAbrv\!\left(\hat{R}(I,I';S_{n}),n \!-\! |I \!\cup\! I'|\right)
\right]
\leq \frac{\delta}{n},
\end{align*}
where the last inequality is due to the empirical Bernstein inequality \cite{MaurerPo09}.

Taking the union bound over the $n$ possible values of $|S_n'|$, we get
\[
\P\!\left[ \err(h_{S'_{n}}) > \BoundAbrv\!\left(\tfrac{n}{n-m} \serr(h_{S'_{n}}), n-m \right) \right] \leq \delta.
\]

Noting that 
$|\I_{n,m}|^{2} = \binom{n}{m/2}^{2} \leq \left(\frac{2 e n}{m}\right)^{m}$, we have that for $Q$ as defined in \eqref{KSUbound},
which is given by
\begin{align*}
Q(n,\alpha,m,\delta) := 
& \frac{n}{n-m} \alpha + 
\sqrt{\frac{8 (\frac{n}{n-m})\alpha\big(m \ln( 2 e n / m ) + \ln(2n/\delta)\big) }{n-m}} 
\\ & + \frac{9\big( m \ln( 2 e n / m ) + \ln(2n/\delta)\big)}{n-m},
\notag
\end{align*}
it holds that
\[
\BoundAbrv\!\left(\tfrac{n}{n-m} \alpha, n-m \right) \leq Q(n, \alpha, m, \delta).
\]
Thus, $Q$ satisfies property $\bQ1$.
Furthermore, property \bQ2 (monotonicity in $\alpha$ and in $m$) can also be easily verified from the 
definition in \eqref{KSUbound}. 
For property \Qthreeb, 
observe that for $m_n = o(n)$ and
a sufficiently large $n$,
$\frac{m_n\log(n/m_n)}{n-m_n} \leq 2\frac{m_n}{n}\log(\frac{m_n}{n})$. Thus,
since $\frac{m_n}{n}\rightarrow 0$, we have
$\frac{m_n\log(n/m_n)}{n-m_n} \rightarrow 0$. \Qthreeb\ is thus satisfied via
any convergent series $\sum_{n=1}^\infty \delta_n < \infty$ such that
$\delta_n= e^{-o(n)}$; note that this requires the decay of $\delta_n$ to be
sufficiently slow. 

We considered above the case of $m \in [n-2]$. From the definition of $Q$, the required properties trivially hold also for larger values of $m$.

\subsection{Proof of \lemref{richness}}
\label{app:richness}
Let $\eta_y: \X \to [0,1]$ be the conditional probability function for label $y\in\Y$,
\beq
\eta_y(x) = \P( Y = y \gn X = x),
\eeq
which is measurable by \cite[Corollary B.22]{MR1354146}.
Define $\teta_y: \X \to [0,1]$ as $\eta_y$'s conditional expectation function with respect to $(\SP,\mss)$: For $x$ such that \mbox{$I_\SP(x) \cap\mss \neq \emptyset$},
\beq
\teta_y(x) = \P(Y = y \gn X \in I_\SP(x) \cap\mss)
= \frac{\int_{I_\SP(x) \cap\mss} \eta_y(z) \diff\mu(z) }{\mu(I_\SP(x) \cap\mss)}.
\eeq
For other $x$, define $\teta_y(x) = \pred{y\text{ is lexicographically first}}$.
Note that $(\teta_y)_{y\in\Y}$ are piecewise constant on the cells of the restricted partition~$\SP \cap\mss$.
By definition, the Bayes classifier $h^*$ and the true majority-vote classifier $h^*_{\SP,\mss}$ satisfy
\beq
h^*(x) &=& \argmax_{y\in\Y} \eta_y(x),
\\
h^*_{\SP,\mss}(x) &=& \argmax_{y\in\Y} \teta_y(x).
\eeq
It follows that
\beq
&&\P( h^*_{\SP,\mss}(X) \neq Y  \gn X = x) - \P( h^*(X) \neq Y  \gn X = x)
\\
&&= \eta_{h^*(x)}(x) - \eta_{h^*_{\SP,\mss}(x)}(x) 
\\ && =  \max_{y\in\Y} \eta_y(x) - \max_{y\in\Y} \teta_y(x)
\\
&& \leq \max_{y\in\Y} |\eta_y(x) - \teta_y(x) |.
\eeq
By condition \myi\ in the lemma statement, $\mu(\X\setminus\mss)\leq \nu $.
Thus,
\beqn
\nonumber
\err(h^*_{\SP,\mss}) - R^* & = &
\P( h^*_{\SP,\mss}(X) \neq Y) - \P( h^*(X) \neq Y)
\\
\nonumber
& \leq & \mu(\X\setminus\mss) +  \int_{\mss} \max_{y\in\Y}|\eta_y(x) - \teta_y(x)| \diff\mu(x)
\\
\nonumber
& \leq & \nu +  \sum_{y\in\Y}\int_{\mss} |\eta_y(x) - \teta_y(x)| \diff\mu(x). 
\eeqn
Let $\Y_\nu \subseteq \Y$ be a finite set of labels such that $\P[Y \in \Y_\nu] \geq 1-\nu$. Then 
\beqn
\label{eq:err_R_eta_teta}
\err(h^*_{\SP,\mss}) - R^* \leq 2\nu +  \sum_{y\in\Y_\nu}\int_{\mss} |\eta_y(x) - \teta_y(x)| \diff\mu(x).
\eeqn
To bound the integrals in (\ref{eq:err_R_eta_teta}), we approximate $(\eta_y)_{y\in\Y}$ with functions from the dense set of Lipschitz functions, applying \lemref{dense_cont} above.
Since $\eta_y \in L^1(\mu)$ for all $y\in\Y_\nu$ and $|\Y_\nu|<\infty$,
Lemma~\ref{lem:dense_cont} implies that there are $|\Y_\nu|$ Lipschitz functions $(r_y)_{y\in\Y_\nu}$ such that
\beqn
\label{eq:eta_r_eps}
\max_{y\in\Y_\nu} \int_{\X} |\eta_y(x) - r_y(x)| \diff\mu(x)  \leq \nu/ |\Y_\nu|.
\eeqn
Similarly to $(\teta_y)_{y\in\Y_\nu}$, define the piecewise constant functions $(\tr_y)_{y\in\Y_\nu}$ by
\beq
\tr_y(x) = \E[r_y(X) \gn X \in I_\SP(x) \cap\mss ]
= \frac{{\ds\int_{I_\SP(x) \cap\mss} r_y(z) \diff\mu(z)} }{\mu(I_\SP(x) \cap\mss)}.
\eeq
We bound each integrand in (\ref{eq:err_R_eta_teta}) by
\beqn
\nonumber
&& |\eta_y(x) - \teta_y(x)| 
\\ 
\label{eq:abs_chain}
&& \quad \leq |\eta_y(x) - r_y(x)| + |r_y(x) - \tr_y(x)| 
+| \tr_y(x) - \teta_y(x)|.
\eeqn
The integral of the first term in (\ref{eq:abs_chain}) is smaller than $\nu/|\Y_\nu|$ by the definition of $r_y$ in (\ref{eq:eta_r_eps}).
For the integral of the third term in (\ref{eq:abs_chain}),
\beq
&& \int_{\mss} \abs{ \tr_y(x) - \teta_y(x)} \diff\mu(x)
\\
&&= \sum_{\sp\in\SP} 
\abs{\E[r_y(X) \pred{X \in \sp \cap\mss}] - \E[\eta_y(X)\pred{X \in \sp \cap\mss}]}
\\
&&=
\sum_{\sp\in\SP} \abs{\int_{\sp \cap\mss} r_y(x)\diff\mu(x) - \int_{\sp \cap\mss} \eta_y(x)\diff\mu(x)}
\\
&&=
\sum_{\sp\in\SP} \abs{\int_{\sp \cap\mss} (r_y(x)- \eta_y(x)) \diff\mu(x)}
\\
&&\leq  \int_{\mss} \abs{r_y(x) - \eta_y(x)} \diff\mu(x) \;\leq\; \nu/|\Y_\nu|.
\eeq
Finally, for the integral of the second term in (\ref{eq:abs_chain}),
we denote
\beq
\bSP = \set{\sp \cap \mss: \mu(\sp \cap \mss)\neq 0, \sp\in\SP}
\eeq
and note that
\beq
&& \int_{\mss} |r_y(x) - \tr_y(x)|\diff\mu(x)
\\
&& \quad = 
\sum_{\bsp\in\bSP} \int_{\bsp} \abs{r_y(x) - \frac{\E[r_y(X) \pred{X \in \bsp}]}{\mu(\bsp)}} \diff\mu(x)
\\
&& \quad = 
\sum_{\bsp\in\bSP}\frac{1}{\mu(\bsp)} 
\int_{\bsp} \abs{r_y(x) \mu(\bsp) - \E[r_y(X) \pred{X \in \bsp}] }
\diff\mu(x)
\\
&& \quad = 
\sum_{\bsp\in\bSP}\frac{1}{\mu(\bsp)} 
\int_{\bsp} \abs{r_y(x) \int_{\bsp}\diff\mu(z) -
\int_{\bsp} r_y(z) \diff\mu(z)}
\diff\mu(x)\\
&& \quad = 
\sum_{\bsp\in\bSP}\frac{1}{\mu(\bsp)} 
\int_{\bsp} \abs{ \int_{\bsp}(r_y(x) - r_y(z))\, \diff\mu(z)}
\diff\mu(x)
\\
&& \quad \leq 
\sum_{\bsp\in\bSP}
\frac{1}{\mu(\bsp)} 
\int_{\bsp} \int_{\bsp} \abs{r_y(x) -  r_y(z)} \diff\mu(x) \diff\mu(z).
\eeq
Since $|\Y_\nu|<\infty$ and any Lipschitz function is uniformly continuous on all of $\X$, the finite collection $\set{r_y:y\in\Y_\nu}$ is equicontinuous.
Namely, there exists a diameter $\beta = \beta(\nu)> 0$ such that for any $A\subseteq \X$ with $\diam(A) \leq \beta$,
$$\max_{y\in\Y_\nu}\abs{r_y(x) - r_y(z)} \leq \nu/|\Y_\nu|$$
for every $x,z\in A$  (note that $\beta(\nu)$ does not depend on $(\SP,\mss)$).
By condition \myii\ in the lemma statement, $\diam(\sp \cap\mss) \leq \beta$ for all $\sp\in\SP$.
Hence,
\beq
\frac{1}{\mu(\sp \cap\mss)} \int_{\sp\cap\mss} \int_{\sp\cap\mss} \abs{r_y(x) -  r_y(z)} \diff\mu(x) \diff\mu(z) \leq \frac{\nu}{|\Y_\nu|} \mu(\sp \cap\mss).
\eeq
Summing over all cells $\sp\in\SP$ with $\mu( \sp \cap\mss) \neq 0$, the integral of the second term in (\ref{eq:abs_chain}) satisfies
\beq
\int_{\mss} |r_y(x) - \tr_y(x)|\diff\mu(x) \leq \nu/|\Y_\nu|.
\eeq
Combining the bounds for the three terms,
\beq
\sum_{y\in\Y_\nu}\int_{\mss} |\eta_y(x) - \teta_y(x)| \diff\mu(x)
\leq 
\sum_{y\in\Y_\nu} \frac{3\nu}{|\Y_\nu|} = 3\nu.
\eeq
Applying this bound to (\ref{eq:err_R_eta_teta}), we conclude
\(
\err(h^*_{\SP,\mss}) - R^* \leq 5\nu.
\)

\subsection{Proof of \lemref{boundqd}}
\label{ap:boundqd}
Let $\bm i = \bm i(\g) \in [n]^d$ be the set of indices in the net $\bm X = \bm X(\gamma)$ selected by the algorithm. %
Let ${\bm Y}^* \in \Y^d$ be the true majority-vote labels with respect to the restricted partition $\Vor(\bm X)\cap\UB_{2\g}(\bm X)$,
\beqn
\label{eq:Y_star}
({\bm Y}^*)_j = 
y^*(V_j \cap \UB_{2\g}(\bm X)),
\qquad j\in[d].
\eeqn
We pair $\bm X$ with the labels $\bm Y^*$ to obtain the labeled set
\beqn
\label{eq:Strue}
S_n(\bm i, *) := S_n(\bm i, {\bm Y}^*) = 
(\bm X,\bm Y^*)
\in (\X\times\Y)^d.
\eeqn
Note that conditioned on $\bm X$, $S_n(\bm i, *)$ does not depend on the rest of $S_n$. 

The induced $1$-NN classifier $h_{S_n(\bm i, *)}(x)$ can be expressed as 
$h_{\SP,\mss}^*(x) = y^*(I_\SP(x) \cap\mss)$
with $\SP = \Vor(\bm X)$
and $\mss=\UB_{2\g}(\bm X)$
(see \eqref{Strue2} for the definition of $h_{\SP,\mss}^*$). We now show that 
\beqn
\label{eq:err_R}
\missmass_\g(\bm X_n) \leq \frac{\eps}{10}
\quad \implies \quad
\err(h_{S_n(\bm i, *)})\leq R^* + \eps/2,
\eeqn
by showing that under the assumption $\missmass_\g(\bm X_n) \leq \frac{\eps}{10}$, the conditions of \lemref{richness} hold for $\Vor,\mss$ as defined above.
To this end, we bound the diameter of the partition $\Vor \cap\mss = \Vor \cap \UB_{2\g}(\bm X)$, and the measure of the missing mass $\mu(\X\setminus\mss) = \missmass_{2\g}(\bm X)$ under the assumption.

To bound the diameter of the partition $\Vor \cap \UB_{2\g}(\bm X)$, let $x \in V_j \cap  \UB_{2\g}(\bm X)$. Note that $V_j$ is the Voronoi cell centered at $x_{i_j} \in \bm X$. Then $\rho(x,x_{i_j}) = \min_{i \in \bm i} \rho(x,x_i)$ and, since $x \in \UB_{2\g}(\bm X)$, $\min_{i \in \bm i} \rho(x,x_i) \leq 2\gamma$. Therefore
\beq
\diam(\SP \cap\mss) = \max_j\diam(V_j \cap \UB_{2\g}(\bm X))\leq 4\g.
\eeq
To bound $\missmass_{2\g}(\bm X)$ under the assumption $\missmass_\g(\bm X_n) \leq \frac{\eps}{10}$,
observe that for all $z \in \UB_{\g}(\bm X_n)$,
there is some $i \in [n]$ such that $z \in B_\g(x_i)$.
For this $i$, there is some $j \in \bm i$ such that $x_i \in B_\g(x_j)$,
since $\bm X$ is a $\g$-net of $\bm X_n$. Therefore $z \in B_{2\g}(x_j)$.
Thus, $z \in \UB_{2\g}(\bm X)$. It follows that
$\UB_\g(\bm X_n) \subseteq \UB_{2\g}(\bm X)$, thus $\missmass_{2\g}(\bm X) \leq \missmass_\g(\bm X_n)$.
Under the assumption, we thus have $\missmass_{2\g}(\bm X) \leq \frac{\eps}{10}$.
Hence, by the choice of $\g=\g(\eps)$ in the statement of the lemma, \lemref{richness} implies \eqref{err_R}.

To bound $Q_n(\a_n(\g), \rns_n(\g))$, we consider the relationship between the hypothetical true majority-vote classifier $h_{S_n(\bm i, *)}$ and the actual classifier returned by the algorithm, $h_{S_n(\bm i, \bm \tbY)}$. Note that 
\beq
\a_n(\g) 
= \serr(h_{S_n(\bm i, \bm \tbY)}, S_n)
= \min_{\bm Y \in\Y^d}\serr(h_{S_n(\bm i, \bm Y)}, S_n)
\leq  \serr(h_{S_n(\bm i, *)}, S_n),
\eeq
and thus, from the monotonicity Property \bQ2 of $Q$, 
\beqn
Q_n(\a_n(\g), \rns_n(\g))
\leq
\label{eq:Q_Q}
Q_n(\serr(h_{S_n(\bm i, *)}, S_n), \rns_n(\g)).
\eeqn
Combining (\ref{eq:err_R}) and (\ref{eq:Q_Q}) we have that
\beq
&&
\left\{Q_n(\a_n(\g),\rns_n(\g))
>
R^* + \eps
\;\;\wedge\;\;
\missmass_\g(\bm X_n) \leq \frac{\eps}{10}
\;\;\wedge\;\;
\rns_n(\g)=d
\right\}
\\
&& \qquad \;\implies\;
\left\{ Q_n(\serr(h_{S_n(\bm i, *)}, S_n), d) > \err(h_{S_n(\bm i, *)}) + \frac{\eps}{2}  \;\;\wedge\;\; |\bm i|=d \right\}.
\eeq
Hence, for all $d \leq \Ng$,
\beqn
\nonumber
p_d &\leq& \P\left[
Q_n( \serr(h_{S_n(\bm i, *)}, S_n),d )
>
\err(h_{S_n(\bm i, *)}) + \frac{\eps}{2}
\;\wedge\;
|\bm i| = d
\right]
\\
\label{eq:sum-bound-lb}
&\leq& 
\P\left[
 \exists \bm i \in [n]^d
 : Q_n( \serr(h_{S_n(\bm i, *)}, S_n),d ) 
>
\err(h_{S_n(\bm i, *)}) +  \frac{\eps}{2}
\right].
\eeqn

To bound the last expression, let $\bm i\in [n]^d$ and denote
\beq
r_{d,n} = \sup_{\alpha \in (0,1)} (Q_n(\alpha,d) - \alpha).
\eeq
We thus have,
\beq
Q_n(\serr(h_{S_n(\bm i, *)},S_n) , d) \leq \serr(h_{S_n(\bm i, *)}, S_n) + r_{d,n}.
\eeq
Let $\bm i' = \{1,\dots,n\}\setminus \bm i$ and note that
\beq
\serr(h_{S_n(\bm i, *)}, S_n) \leq \frac{n-d}{n} \serr(h_{S_n(\bm i, *)}, S_n(\bm i')) + \frac{d}{n}.
\eeq
Combining the two inequalities above, we get
\[
Q_n(\serr(h_{S_n(\bm i, *)},S_n) , d) \leq \serr(h_{S_n(\bm i, *)}, S_n(\bm i')) + \frac{d}{n} + r_{d,n}.
\]
Recalling $\Ng\in o(n)$, 
by Property \Qthreeb,
\beq
\lim_{n\to\infty} \frac{\Ng}{n} + r_{\Ng,n} =0.
\eeq
In addition, by \bQ2,
we have
$r_{d,n} \leq r_{\Ng,n}$ for all $d\leq \Ng$.
Hence, we take $n$ sufficiently large so that for all $d\leq \Ng$,
\beq
\frac{d}{n}  +   r_{d,n}
 \leq \frac{\eps}{4},
\eeq 
and thus
\[
Q_n(\serr(h_{S_n(\bm i, *)},S_n) , d) \leq \serr(h_{S_n(\bm i, *)}, S_n(\bm i')) + \frac{\eps}{4} .
\]
Therefore, for such an $n$, 
\beq
&&
Q_n( \serr(h_{S_n(\bm i, *)}, S_n),d ) 
>
\err(h_{S_n(\bm i, *)}) +  \frac{\eps}{2}
\\
&& \qquad\qquad\implies\quad
\serr(h_{S_n(\bm i, *)}, S_n(\bm i')) 
>
\err(h_{S_n(\bm i, *)})  +  \frac{\eps}{4}
.
\eeq
Now,
\beqn
\label{eq:exp-bounds}
&& \P\left[
\serr(h_{S_n(\bm i, *)}, S_n(\bm i')) 
>
\err(h_{S_n(\bm i, *)})   +  \frac{\eps}{4}
\right]
\\
\nonumber
& = &
\E_{S_n(\bm i)}
\left[
\P_{S_n(\bm i') \gn S_n(\bm i)}
\left[
\serr(h_{S_n(\bm i, *)}, S_n(\bm i')) 
>
\err(h_{S_n(\bm i, *)})  +  \frac{\eps}{4}
\right]
\right]
.
\eeqn
Since $\P_{S_n(\bm i') \gn S_n(\bm i)}$ is a product distribution, by Hoeffding's inequality we have that (\ref{eq:exp-bounds}) is bounded above by $e^{-2(n-d)(\frac{\eps}{4})^2}$. 
Since $h_{S_n(\bm i, *)}$ is invariant to permutations of $\bm i$'s entries,
bounding (\ref{eq:sum-bound-lb}) by a union bound over $\bm i$ yields
\beq
p_d \leq
\binom{n}{d}
e^{-2(n-d)(\frac{\eps}{4})^2}
\leq
e^{d\log\left(\frac{en}{d}\right)  -2(n-d)(\frac{\eps}{4})^2},
\eeq
where we used
$\binom{n}{d} \leq \left(\frac{en}{d}\right)^d$.
Selecting $n$ large enough so that for all $d\leq \Ng$ we have $d\log(e n/d) \leq (n-d)(\frac{\eps}{4})^2$ and $d \leq n/4$, we get the statement of the lemma.

\subsection{Proof of \lemref{sublinear_comp}}
Let $\gnet(\g)$ be any $\g$-net of $S_n=(X_1,\dots,X_n)$ and let $\BP= \{\bp_1,\bp_2,\dots\}$ be a fixed countable partition of $\X'$ (for separable $\X' \subseteq \X$ of $\mu(\X')=1$) with 
\beq
\diam(\BP) = \sup_{i\in\N}(\diam(\bp_i)) < \g,
\eeq
which exists by the separability assumption.
Denote the number of occupied cells in $\BP$ by
\beq
U_n(X_1,\dots,X_n) = \sum_{\bp_i\in\BP} \pred{S_n \cap \bp_i \neq \emptyset}.
\eeq
Since $\gnet(\g)$ is a $\g$-net of $S_n$, any cell $\bp_i\in\BP$ contains at most one $X\in\gnet(\g)$.
Hence,
\beq
|\gnet(\g)| \leq U_n(X_1,\dots,X_n).
\eeq
So it suffices to bound $U_n$.
To this end, denote by $i(X)$ the cell in $\BP$ such that  $X\in A_{i(X)}$.
Then,
\beq
\E[U_n(X_1,\dots,X_n)] &=& \sum_{j=1}^{n} \P
\left[
A_{i(X_j)} \cap \{X_{1}\dots,X_{j-1}\} = \emptyset
\right]
\\
&=& 
\sum_{j=1}^{n} \E[1-\mu(\cup_{k=1}^{j-1} A_{i(X_k)})].
\eeq
Since 
\begin{equation}
\lim_{j\to\infty}
\E\sqprn{\mu\paren{\bigcup_{k=1}^{j} A_{i(X_k)}}}
= \lim_{n\to\infty} \E\sqprn{\mu\paren{\bigcup_{\bp_i\in\BP:\bp_i\cap S_n\neq\emptyset} \bp_i}} = 1,\label{eq:missmass0}
\end{equation}
we have
\beq
\E[U_n(X_1,\dots,X_n)] \in o(n).
\eeq
Let
\beq
\Ng := 2\left(\E[U_n(X_1,\dots,X_n)] + \sqrt{n\log n}\right) \in o(n).
\eeq
Since $U_n$ is $1$-Lipschitz with respect to the Hamming distance, McDiarmid's inequality implies
\beq
\P[2U_n(X_1,\dots,X_n) \geq \Ng] \leq 1/n^2,
\eeq
concluding the proof.

\subsection{Proof of \lemref{missing_mass}}
Let $\BP= \{\bp_1,\bp_2,\dots\}$ be a fixed countable partition of 
$\X$ with $\diam(\BP)<\g$
as in the proof of Lemma \ref{lem:sublinear_comp}.
Consider the random variable
\beq
F_{\BP}(S_n)=1-\mu(\cup_{\bp_i\in\BP:\bp_i\cap S_n\neq\emptyset} \bp_i),
\eeq
corresponding to the total mass of all cells not hit by the sample $S_n$.
Since $\diam(\BP) < \g$, we have that
$\bp_i\subseteq B_\g(x)$
for all $x\in \bp_i$.
Hence, with probability $1$
\beq
\missmass_\g(S_n) = 1-\mu(\UB_\g(S_n)) \leq F_{\BP}(S_n),
\eeq
whence
\beq
\P\left[
  \missmass_\g(S_n) \ge
\E[F_{\BP}(S_n)]  
  + t
\right]
&\leq &
\P\left[
F_{\BP}(S_n)
\geq 
\E[F_{\BP}(S_n)]
+ t
\right].
\eeq
Invoking the concentration bound for the missing mass in \cite[Theorem 1]{berend2013concentration},
\beq
\P\left[
F_{\BP}(S_n)
\geq 
\E[F_{\BP}(S_n)]
+ t
\right]
&\leq&
\exp\left(- n t^2\right)
\eeq
and observing that,
by
(\ref{eq:missmass0}),
$\lim_{n\to\infty}\E[F_{\BP}(S_n)]=0$,
shows that the choice $u_\g(n):=\E[F_{\BP}(S_n)]$
verifies the properties claimed.

\section{Auxiliary lemmas for Section \ref{SEC:UBC_impossible} -- case (\textit{I})}
\label{ap:nonsep_I}

\subsection{Auxiliary Lemma \ref{lem:borel_subset}}

\begin{lemma}
\label{lem:borel_subset}
Suppose that $U$ is a discrete subset of a metric space $(\X,\rho)$.
Then every $E\subseteq U$ is Borel.
\end{lemma}
\bepf
By discreteness, for each $x\in U$ there is an $r_x>0$ such that $B_{r_x}(x)\cap U=\set{x}$. The latter property is satisfied by any other $0<r<r_x$. Further, for any $n\in\N$, the set $V_n:=\cup_{x\in E}B_{r_x/n}$ is open and
\beq
E = \bigcap_{n\in\N} V_n,
\eeq
whence $E$ is Borel.
\enpf

\subsection{Proof of Lemma \ref{lem:nonsep-eps-sep}}
For $\eps>0$, consider the family of all $\eps$-separated subsets of $\Uorel$,
\beq
\mathscr{F}_\eps = \set{F\subseteq \Uorel: \Delta(A,B)\geq\eps, A\neq B\in F}.
\eeq
Let $\mathscr{F}_\eps^{\max}$ be the set of maximal elements in $\mathscr{F}_\eps$, which by Zorn's lemma is non-empty. 
Any $F\in \mathscr{F}_\eps^{\max}$ is an $\eps$-net of $\Uorel$ by maximality.
Let $\{\eps_i\}_{i\in\N}$ be a sequence such that $\eps_i>0$ and $\lim_{i\to\infty}\eps_i =0$ and let $\{D_i\}_{i\in\N}$ be such that $D_i\in \mathscr{F}_{\eps_i}^{\max}$ for all $i\in\N$.
Clearly, the set $D =\cup_{i\in\N} D_i$ is dense in $\Uorel$.
Moreover, by \cite[Lemma 2]{barbati1997density}, 
\beq
|D| = \sup_{i\in\N} |D_i| = d(\Uorel),
\eeq
where $d(\Uorel)$ is the density of $(\Uorel,\Delta)$ (namely, the smallest cardinality of any subset of $\Uorel$ which is dense for the metric space).
Hence, for any cardinal $\alpha < d(\Uorel)$, there is a finite $i\in\N$ such that $D_i$ is an $\eps_i$-separated set with cardinality $|D_i|\geq \alpha$.
Thus, to prove the Lemma it suffices to show that $d(\Uorel)>\kappa$.

To show that $d(\Uorel)>\kappa$, consider the measure algebra $(\Uorel,\tilde\mu)$.
Since $\tilde\mu$ is totally finite, the topology of the measure algebra is the same topology generated by the metric space $(\Uorel,\Delta)$ \cite[323A(d)]{fremlin2000measure}.
So the density of the measure algebra topology is $d(\Uorel)$ as well.
The {Maharam type} $\tau(\Uorel)$ of $(\Uorel,\tilde\mu)$  is defined as the smallest cardinality of any subset of $\Uorel$ which generates the topology of $(\Uorel,\tilde\mu)$ \cite[331E-F]{fremlin2000measure}.
Since $\Uorel$ is infinite,
\cite[521O(ii)]{fremlin2000measure} implies $d(\Uorel) = \tau(\Uorel)$.
Since $\mu$ is a finite, atomless, and $\kappa$-additive measure on $(\X,2^\X)$, Gitik-Shelah Theorem \cite[543F]{fremlin2000measure} implies that $(\Uorel,\tilde\mu)$ has Maharam type 
$\tau(\Uorel)>
\kappa$.
Hence, $d(\Uorel)=\tau(\Uorel)>\kappa$.

\section{Auxiliary lemmas for Section \ref{SEC:UBC_impossible} -- case (\textit{II})}
\label{ap:nonsep_II}

\subsection{Auxiliary Lemmas \ref{lem:measure_1_intersect} and \ref{lem:two_valued_properties}}

\begin{lemma} Let $\mu$ be a measure. 
\label{lem:measure_1_intersect}
For any countable family $\{U_i\}_{i=1}^\infty$ with $\mu(U_i)=1,\forall i\in\N$, we have $\mu\left(\bigcap_{i=1}^\infty U_i\right) = 1$.
\end{lemma}

\begin{proof}[Proof of Lemma \ref{lem:measure_1_intersect}]
Let $U_i^c = \X\setminus U_i$ and $V_i = U_i^c\setminus\{\bigcup_{j<i} U_j^c\}$. 
Then $\mu(V_i)=0$, $\{V_i\}_{i=1}^\infty$ are pairwise disjoint, and $\left(\bigcap_{i=1}^\infty U_i \right)^c = \bigcup_{i=1}^\infty U_i^c  = \bigcup_{i=1}^\infty V_i$.
Thus, $\mu\left(\bigcap_{i=1}^\infty U_i\right) = 1 - \mu\left(\bigcup_{i=1}^\infty V_i\right) = 1 - \sum_{i=1}^\infty \mu(V_i)=1$.
\end{proof}

\begin{lemma}
\label{lem:two_valued_properties}
Let $\nu\neq\mu$ be any distinct two-valued measures on $(\X,\Borel)$.
Then there exists $B\subseteq \X$ such that $\nu(B) = \mu(\X\setminus B) = 1$.
\end{lemma}


\begin{proof}[Proof of Lemma \ref{lem:two_valued_properties}]
By definition, $\mu$ and $\nu$ are distinct if $\exists A\subseteq\X$ such that $\mu(A)\neq\nu(A)$. Since $\mu$ and $\nu$ are two-valued, we must have that $\mu(A) = 1 - \nu(A)\in\{0,1\}$. In addition, either $\mu(A)=1$ or $\mu(\X\setminus A)=1$.
Assuming without loss of generality that $\mu(A)=1$, the set $B=\X\setminus A$ satisfies the required properties.
\end{proof}

\subsection{Proof of Lemma \ref{lem:homogeneous}}

The required measure $\mu$ is taken as a witnessing measure with the additional property of being \emph{normal}.
\begin{remark}
  In the terminology of \cite[Chapter 10]{jech}, a two-valued witnessing measure $\mu$ on $\X$ is equivalent to a $\kappa$-complete non-principal ultrafilter on $2^\X$ consisting of all sets with measure $1$ under $\mu$.
  The latter
  is normal if the ultrafilter is also closed under diagonal intersection.
  \eor
\end{remark}
For our needs, it suffices that a normal measure $\mu$ exists on $\MSk$, a fact proved in \cite[Theorem 10.20]{jech}.
To establish the homogeneity property of $\mu$ we apply \cite[Theorem 10.22]{jech}, which in the terminology of current paper takes the following form.

\begin{theorem}[\cite{jech}, Theorem 10.22]
\label{lem:homogeneous_jech}
Let $\X$ be of two-valued measurable cardinality $\kappa$, let $\mu$ be a normal measure on $\MSk$, and let $F:[\X]^{<\omega}\to \mathcal{R}$ with $|\mathcal{R}|<\kappa$. 
Then, there exists a set $U\subseteq\X$ with $\mu(U)=1$ that is homogeneous for $F$.
\end{theorem}
Since $\kappa$ is two-valued measurable, Ulam's dichotomy implies $\kappa>\cont = |\R|$.
An application of Theorem \ref{lem:homogeneous_jech} with $\mathcal{R}=\R$
completes the proof of the Lemma.

\subsection{Proof of Lemma \ref{lem:bayes_optimal}}
Assume first that $\phi=\mu$ and let $(X,Y)\sim\mix_{\mu}$.
Note that $Y\sim \text{Bernoulli}(2/3)$ and $X\sim\mu$ is independent of $Y$.
Thus, for any classifier $h:\X\to\{0,1\}$,
\beq
\err_{\mix_{\mu}}(h) &=& 
\P[h(X)\neq Y] 
\\
&=&
\frac{2}{3}\cdot\P[h(X) = 0 \gn Y=1]
+ \frac{1}{3}\cdot \P[h(X) = 1 \gn Y=0]
\\
&=&
\frac{2}{3}\cdot\mu(h(X)=0) + 
\frac{1}{3}\cdot \mu(h(X)=1)
\\
&=& \frac{2}{3} - \frac{1}{3} \cdot\mu(h(X)=1)
\geq \frac{1}{3}.
\eeq
Hence, the Bayes-optimal error is $1/3$ (as demonstrated by the classifier $\hst(x)=1$) and is achieved if and only if $\mu(h(X)=1)=1$.

Assume now that $\phi=\nu\neq\mu$.
Then,
\beq
X|Y
\;\sim\;
\begin{cases}
\nu, & \text{if } Y=1;
\\
\mu, & \text{if } Y=0.
\end{cases}
\eeq
Thus, the error of a classifier $h$ is
\beq
\err_{\mix_{\nu}}(h) &=& 
\frac{2}{3}\cdot\P[h(X) = 0 \gn Y=1]
+ \frac{1}{3}\cdot \P[h(X) = 1 \gn Y=0]
\\
&=&
\frac{2}{3}\cdot\nu(h(X)=0) + 
\frac{1}{3}\cdot\mu(h(X)=1).
\eeq
Since both $\mu$ and $\nu$ are two-valued, Lemma \ref{lem:two_valued_properties} implies $\exists B\subseteq\X$ such that
\beq
\nu(B) = \mu(\X\setminus B) = 1.
\eeq
Thus, the Bayes-optimal error is $0$ (as demonstrated by $\hst(x) = \pred{x \in B}$) and is achieved if and only if $\nu(h(X)=0)=0$ and $\mu(h(X)=1)=0$.

\section{Total ordering in metric spaces}
\label{ap:total_order}
The following is due to Vladimir Pestov (via personal communication).

\begin{proposition}
\label{prop:total_order}
Every metric space $X$ admits a total order $\prec$ with the property that the graph of this order in $X\times X$ is Borel measurable, in particular, each initial segment is Borel measurable.
\end{proposition}

\begin{proof}
Let $\tau$ be the weight of $X$, that is, the smallest cardinality of a base.

For a cardinal $\tau$, denote $B(\tau)$ the generalized Baire space of weight $\tau$, that is, a countable topological product of copies of a discrete space of cardinality $\tau$:
\[B(\tau)=\tau_{discrete}^\omega.\]

An easy argument, using standard tools of descriptive set theory, shows that $X$ is Borel isomorphic to a subspace of $B(\tau)$. See e.g. lemma 3.3 in \cite{stone1962non}, although the lemma is establishing a much stronger result than that. (The lemma is about complete metric space $X$, but clearly the conclusion for arbitrary spaces follows by forming a completion first.)

It is enough to construct a Borel measurable order on $B(\tau)$. We will in fact construct a (strict) order which has an open graph. It is a  lexicographic order with regard to any total ordering on $\tau$, e.g., the canonical minimal well-ordering. Namely, an element $x=(x_n)$ is less than $y=(y_n)$, that is, $x\prec y$, if and only if $x_k<y_k$, where
\[k=\min\{n\colon x_n\neq y_n\}.\]
We will show that the graph of $\prec$,
\[\Gamma = \{(x,y)\in B(\tau)^2\colon x\prec y\},\]
is an open set in the topology of $B(\tau)$, thus finishing the argument. Let $(x,y)\in\Gamma$, that is, $x\prec y$. Define $k$ as above. Then $x,y$ can be written as $x=(x_1,\ldots,x_{k-1},x_k, \ldots)$, $y=(x_1,x_2,\ldots,x_{k-1},y_k,\ldots)$, where $x_k<y_k$. The cylinders
\[C_1 = \{z\in B(\tau)\colon z_1=x_1,\ldots,z_{k-1}=x_{k-1},z_k=x_k\}\]
and
\[C_2 = \{z\in B(\tau)\colon z_1=x_1,\ldots,z_{k-1}=x_{k-1},z_k=y_k\}\]
are open in the product topology on $B(\tau)$, so their product is open in $B(\tau)^2$, and also clearly $(x,y)\in C_1\times C_2$, and $C_1\times C_2\subseteq \Gamma$ (as each element of $C_1$ is strictly less than each element of $C_2$). 
\end{proof}

\bibliographystyle{imsart-number}

\bibliography{refs}

\end{document}